\documentclass{article}
\usepackage{iclr2027_conference,times}

\usepackage{amsmath,amssymb,amsfonts,bm}
\usepackage{algorithm}
\usepackage[noend]{algpseudocode}
\usepackage{booktabs}
\usepackage{graphicx}
\usepackage{microtype}
\usepackage{xspace}
\usepackage{amsthm}
\usepackage{hyperref}
\usepackage{placeins}
\usepackage{wrapfig}
\usepackage{enumitem}
\usepackage{url}
\usepackage{tablefootnote}
\usepackage[most]{tcolorbox}
\hypersetup{hidelinks}
\usepackage{xcolor}
\definecolor{gainred}{HTML}{A64B4B}
\newcommand{\gain}[1]{{\textbf{\color{gainred}$+#1$}}}

\usepackage{amsmath,amsfonts,bm}

\def\eqref#1{equation~\ref{#1}}
\def\1{\bm{1}}

\DeclareMathAlphabet{\mathsfit}{\encodingdefault}{\sfdefault}{m}{sl}
\SetMathAlphabet{\mathsfit}{bold}{\encodingdefault}{\sfdefault}{bx}{n}

\newcommand{\E}{\mathbb{E}}

\newcommand{\KL}{D_{\mathrm{KL}}}

\newcommand{\D}{\mathcal{D}}
\newcommand{\ind}{\mathbb{I}}

\newcommand{\pold}{\pi_{\mathrm{old}}}

\newcommand{\pteacher}{\pi_{\text{Teacher}}}

\newcommand{\pji}{\pi_{\mathrm{JI}}}

\newcommand{\mintrl}{{MInTRL}\xspace}

\definecolor{promptolive}{HTML}{BBCC33}
\newtcolorbox[auto counter]{promptbox}[2][]{
    enhanced,
    breakable,
    width=\linewidth,
    colback=promptolive!8!white,
    colframe=black,
    colbacktitle=black,
    coltitle=white,
    fonttitle=\bfseries\small,
    title={Prompt~\thetcbcounter: #2},
    attach boxed title to top left={yshift=-0.1in,xshift=0.15in},
    boxed title style={boxrule=0pt,colframe=black},
    top=8pt,
    bottom=6pt,
    left=7pt,
    right=7pt,
    #1
}

\newtheorem{theorem}{Theorem}[section]
\newtheorem{assumption}[theorem]{Assumption}
\newtheorem{definition}[theorem]{Definition}
\newtheorem{lemma}[theorem]{Lemma}

\title{\mintrl: Off-policy Intervention can boost On-policy RL}

\author{
Mingyu Chen$^{1,2,*}$
\hspace{0.2em}
Yefan Tao$^{1}$
\hspace{0.2em}
Gerald Friedland$^{1}$
\hspace{0.2em}
Xuezhou Zhang$^{2}$
\hspace{0.2em}
Chris Kong$^{1,\dagger}$
\\
$^{1}$Amazon Web Services
\qquad
$^{2}$Boston University
}

\iclrfinalcopy % Uncomment for the camera-ready version, not for submission.
\begin{document}

\raggedbottom

\maketitle

\begingroup
\renewcommand{\thefootnote}{\fnsymbol{footnote}}
\footnotetext[2]{Correspondence to: Chris Kong, \texttt{luyankon@amazon.com}.}
\footnotetext[1]{Work done during an internship at Amazon Web Services.}
\endgroup

\begin{abstract}
Reinforcement learning with verifiable rewards is typically performed on-policy, keeping training data close to the current policy but limiting learning to trajectories that the policy can discover itself.
Off-policy methods such as supervised fine-tuning, on the other hand, can leverage external knowledge beyond the base model’s capabilities, but may suffer from large distribution shift.
The key challenge is thus to expand exploration without sacrificing learnability.
In this work, we introduce \emph{\underline{M}inimal \underline{Int}ervention \underline{R}einforcement \underline{L}earning} (\mintrl), which expands the exploration frontier through sparse, local interventions in otherwise on-policy rollouts. 
During generation, a judge--intervention policy periodically reviews the current policy's output, replaces erroneous suffixes with short corrections, and immediately returns control to the policy. 
During training, \mintrl adopts a sequence-level advantage-regression objective that eliminates the need for importance sampling.
We show that sparse, local interventions can substantially improve coverage beyond finite-budget on-policy sampling while preserving the overall on-policy nature of the resulting trajectories.
Across math and code benchmarks, \mintrl consistently outperforms standard on-policy and off-policy baselines.
Ablations show that \mintrl remains effective with self-intervention and across different judge policies, while performance peaks at moderate intervention intensity, highlighting the importance of intervening minimally.
These results establish minimal intervention as an effective paradigm for enhancing on-policy RL.
\end{abstract}

\section{Introduction}
\label{sec:introduction}

Reinforcement learning with verifiable rewards (RLVR) has become a key
component of post-training large language models (LLMs) for advanced reasoning
\citep{guo2025deepseek,kimi2025k15,yang2025qwen3}. By learning from
self-generated experience with automatically verifiable outcome rewards,
large-scale RL post-training has produced substantial improvements in
mathematical reasoning, code generation, and other challenging reasoning tasks.

Most existing RLVR methods are built around \emph{on-policy} learning, where the policy is optimized using rollouts sampled from the current policy.
While this design avoids behavior-policy mismatch, it also restricts the experience available for learning to what the policy can discover through its own sampling. 
Recent studies suggest that on-policy RLVR often struggles to discover reasoning trajectories beyond those
already reachable by the base model under finite sampling, limiting its ability to expand the model's reasoning capabilities
\citep{chen2026does,wu2025invisible}.

External knowledge can expose the policy to successful reasoning paths that are rarely reached through its own rollouts. 
However, incorporating this knowledge during RL introduces off-policy information, and the resulting trajectories may differ substantially from the current policy's own behavior. 
The challenge is therefore to expand exploration while
preserving the policy's ability to learn from this additional experience.
This raises a fundamental question:
\textbf{can LLMs effectively incorporate external knowledge during RL
training? If so, how much off-policy information should be introduced?}

\begin{figure}[t]
    \centering
    \includegraphics[height=0.245\textwidth]{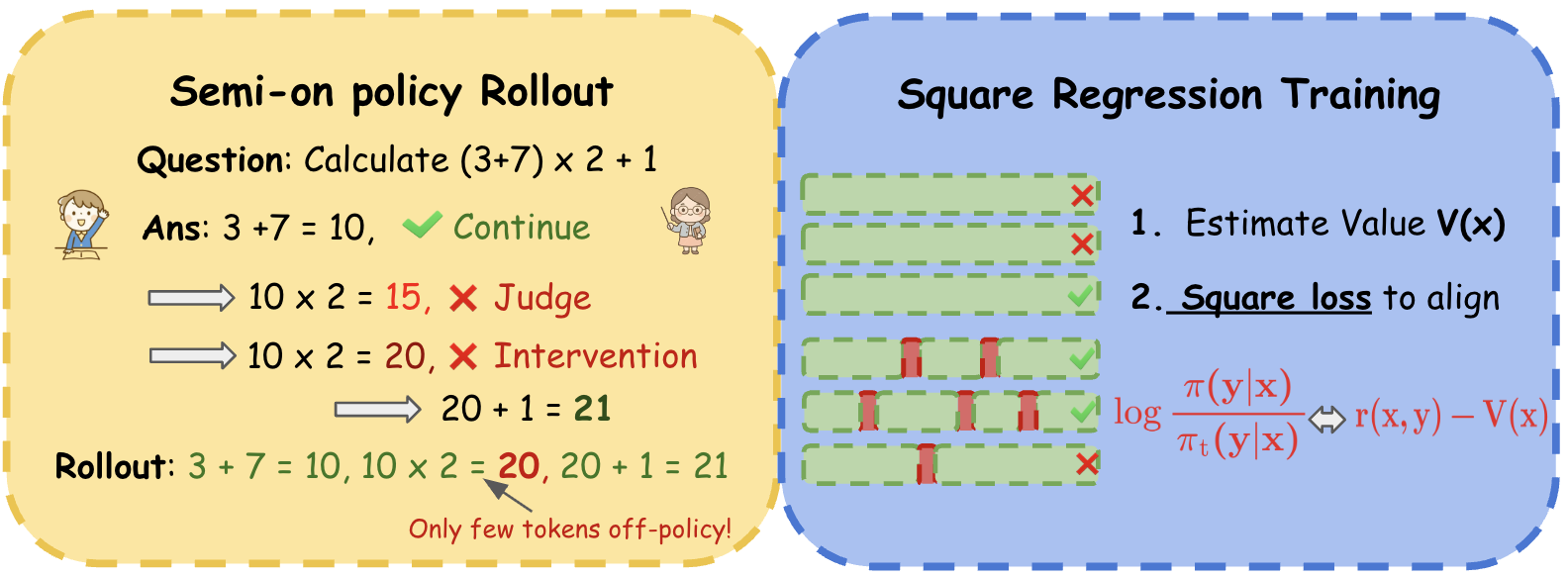}%
    \hfill
    \raisebox{-7.5pt}{%
        \includegraphics[height=0.26\textwidth]{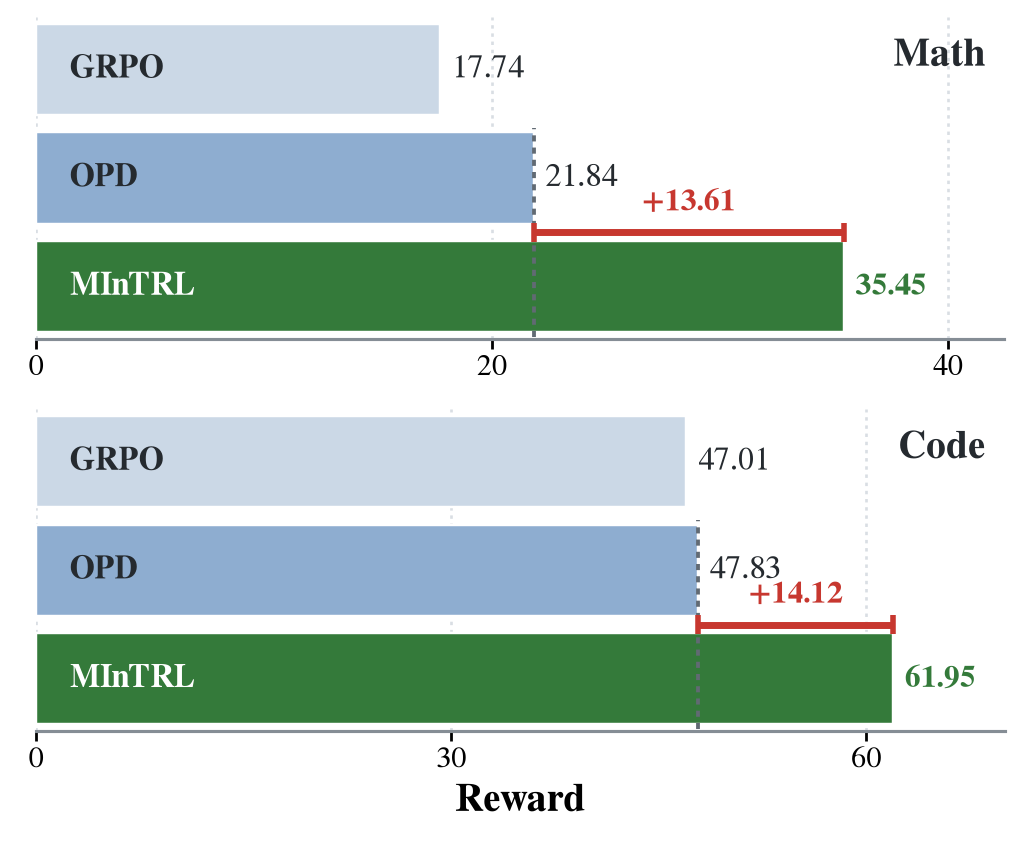}}
    \caption{Overview and main results of \mintrl.}
    \label{fig:overall-pipeline}
\end{figure}

To address this challenge, we propose
\emph{\underline{M}inimal \underline{Int}ervention
\underline{R}einforcement \underline{L}earning} (\mintrl), which incorporates
external knowledge through sparse local interventions while keeping most
of each training trajectory on-policy.
Figure~\ref{fig:overall-pipeline} illustrates the overall pipeline.
During rollouts, the current policy remains in control by default, while
a judge--intervention policy provides short local corrections when errors
are detected before returning control to the current policy.
These interventions help the policy reach successful reasoning paths
that are otherwise unlikely to emerge from its own rollouts, expanding
the experience available for learning with only a small fraction of
off-policy tokens.
To effectively learn from these mixed-provenance trajectories, we further
adopt a regression-based RL objective that avoids behavior-policy
importance weighting.
Our contributions are threefold:
\begin{itemize}[leftmargin=1.3em,topsep=0.3em,itemsep=0.2em]
    \item \textbf{Method.} We introduce \mintrl, a semi-on-policy RL framework
    that uses minimal intervention to expose the current policy to otherwise
    difficult-to-discover reasoning paths, together with a regression objective
    that learns from mixed-provenance trajectories without behavior-policy
    importance weighting.
    \item \textbf{Theory.} We provide a stylized support analysis illustrating a coverage–learnability trade-off under sparse optimal corrections:
    sparse off-policy interventions can expand the set of successful
    trajectories visible under finite sampling, while excessive intervention
    can push those trajectories beyond what the current policy can reliably
    learn from.
    \item \textbf{Experiments.} Across mathematical reasoning and code
    generation with Qwen3-1.7B and Qwen3-4B, \mintrl outperforms on-policy RL,
    distillation, and intervention-based baselines by up to 9.44 percentage
    points over the strongest competitor. 
\end{itemize}

\section{Method}
\label{sec:method}

We now present \mintrl, a framework that incorporates
external guidance through minimal intervention while keeping the resulting
training experience predominantly on-policy.
At each step $t$, \mintrl consists of two stages. During
\emph{semi-on-policy rollout sampling}, trajectories are generated predominantly
by the current policy $\pi_t$, with a judge--intervention policy supplying only
a small number of corrective tokens. During \emph{regression-based RL
training}, we optimize the policy over the collected semi-on-policy rollouts
using a sequence-level regression objective.
We next describe the two stages in detail.

\subsection{Semi-on-policy Rollout Sampling}
\label{sec:rollout-generation}

\begin{wrapfigure}[11]{r}{0.48\linewidth}
    \vspace{-1.4\baselineskip}
    \footnotesize
    \refstepcounter{algorithm}
    \label{alg:intervened-rollout}
    \hrule height 0.8pt
    \kern 2pt
    \noindent\textbf{Algorithm \thealgorithm}
    {Semi-on-policy Rollout Sampling}\par
    \kern 2pt
    \hrule
    \kern 2pt
    \begin{algorithmic}[1]
    \Require Input $x$; policies $\pi_t$ and $\pji$
    \State Initialize the accepted prefix $p\gets\varnothing$.
    \While{the response is incomplete}
        \State Sample $c\sim\pi_t(\cdot\mid x,p)$.
        \State $\pji$ reviews $c$ and returns
               \textsc{Keep} or \textsc{Revise}.
        \State \textbf{\textsc{Keep}:}
               update $p\gets p\Vert c$.
        \State \textbf{\textsc{Revise}:}
               sample $h\sim\pji(\cdot\mid x,p\Vert c_{<e})$.
        \State Update $p\gets p\Vert c_{<e}\Vert h$
               and resume with $\pi_t$.
    \EndWhile
    \State \Return $p$.
    \end{algorithmic}
    \kern 2pt
    \hrule
\end{wrapfigure}

We first describe the \emph{semi-on-policy rollout sampling}
procedure.
Let $x\sim\mathcal{D}$ denote an input prompt,
$\pi_t$ the current policy at iteration $t$, and
$y=(y_1,\ldots,y_T)$ a response with verifiable outcome
reward $r(x,y)$.
Standard on-policy RLVR generates the entire response
from $\pi_t$.
In contrast, our procedure keeps $\pi_t$ in control of
generation while allowing a judge--intervention policy
$\pji$ to contribute a small number of corrective tokens
when needed.
The judge--intervention policy can be either a stronger
teacher model or the same policy model
with privileged context \citep{zhao2026self,shenfeld2026self}.
Algorithm~\ref{alg:intervened-rollout} summarizes
the procedure.

Specifically, for each rollout, $\pi_t$ generates the response
incrementally in short chunks, with each newly generated chunk
reviewed by $\pji$ in the context of the accepted prefix so far.
For example, the first chunk is sampled as
\begin{equation}
c_1=(y_1,\ldots,y_k)
\sim \pi_t(\cdot \mid x).
\end{equation}
The judge-intervention policy $\pji$ then reviews $c_1$ together with the input $x$, deciding whether to keep the chunk or revise it and, in the latter case, identifying the earliest erroneous token:
\begin{equation}
(a_1,e_1)=\mathcal{J}_{\mathrm{JI}}(x,c_1),
\qquad
a_1\in\{\textsc{Keep},\textsc{Revise}\}.
\end{equation}
If $a_1=\textsc{Keep}$, the entire chunk $c_1$ is retained and appended to the accepted prefix.
If $a_1=\textsc{Revise}$, we truncate $c_1$ at the detected error position $e_1$ and preserve only the prefix $(y_1,\ldots,y_{e_1-1})$.
The intervention policy $\pji$ then generates a short corrective continuation $h_1\sim\pji(\cdot\mid x,y_1,\ldots,y_{e_1-1})$ from the retained prefix.
We then append $h_1$ to the retained prefix, yielding $(y_1,\ldots,y_{e_1-1})\Vert h_1$.
After either keeping the original chunk or forming the corrected prefix, control returns to $\pi_t$, which resumes generation from the resulting prefix.
We repeat this generate--review--intervene process until the response is completed.

\iffalse
\begin{figure}[!t]
    \centering
    \begin{minipage}[t]{0.45\linewidth}
        \vspace{0pt}
        {\centering
        \includegraphics[
            height=3.2cm,
            trim=12.5bp 10bp 11bp 13.5bp,
            clip
        ]{figures/mintrl_rollout_pipeline_new.png}\par}
        \vspace{0.5ex}
        \refstepcounter{figure}
        \label{fig:mintrl-rollouts}
        {\footnotesize\raggedright
        Figure~\thefigure:
        \textbf{Semi-on-policy rollout generation.}
        $\pji$ locally revises erroneous chunks before returning control to $\pi_t$.\par}
    \end{minipage}
    \hfill
    \begin{minipage}[t]{0.53\linewidth}
        \vspace{0pt}
        {\centering
        \includegraphics[
            height=3.2cm
        ]{figures/mintrl_proximity_reward.pdf}\par}
        \vspace{0.5ex}
        \refstepcounter{figure}
        \label{fig:mintrl-proximity-reward}
        {\footnotesize\raggedright
        Figure~\thefigure:
        \textbf{Perplexity under $\pi_t$ and rollout reward.}
        Intervention rollouts remain close to on-policy rollouts in perplexity while achieving higher rewards.\par}
    \end{minipage}
\end{figure}
\FloatBarrier
\fi

The key property of this construction is its locality.
The judge-intervention policy generates only a short corrective span, after
which the current policy immediately resumes generation.
Consequently, the resulting trajectory incorporates external information at
the point of error while remaining predominantly generated by $\pi_t$.
%We refer to these mixed-generated trajectories as \emph{semi-on-policy} rollouts.

\subsection{Learning from the Collected Rollouts}
\label{sec:policy-update}

Given the semi-on-policy rollouts generated above, the next question is how to optimize the policy with these trajectories. 
We use the rule-based reward as the learning signal, rather than directly distilling the intervention policy, since the intervention is only intended to steer generation and its local corrections may not always provide reliable supervision. 
A natural first choice is to apply a standard GRPO-style objective. 
However, for these mixed-policy trajectories, doing so requires importance-sampling corrections to account for the behavior-policy mismatch.
Even when interventions affect only a small fraction of tokens, importance ratios can accumulate across the trajectory and become highly unstable.
To avoid this issue, we adopt a \emph{regression-based RL objective} that learns directly from rewarded semi-on-policy trajectories without explicitly dividing by their behavior probabilities.

\paragraph{Regression target.}
Let $\pi_t$ denote the current policy before the update.
  We begin from the standard KL-regularized RL objective
\begin{equation}
    \max_{\pi}\;
    \E_{y\sim\pi(\cdot\mid x)}[r(x,y)]
    -
    \beta\KL\!\left(
        \pi(\cdot\mid x)\,\|\,\pi_t(\cdot\mid x)
    \right).
    \label{eq:mintrl-kl-objective}
\end{equation}
The optimal policy and its normalizing value have the closed forms
\begin{align}
    V^\star_\beta(x)
    &=
    \beta\log
    \E_{y\sim\pi_t(\cdot\mid x)}
    \left[
        \exp\left(r(x,y)/\beta\right)
    \right],
    \label{eq:mintrl-soft-value}\\
    \pi^\star_\beta(y\mid x)
    &=
    \pi_t(y\mid x)
    \exp\left(
        \bigl(r(x,y)-V^\star_\beta(x)\bigr)/\beta
    \right).
    \label{eq:mintrl-optimal-policy}
\end{align}
Equivalently, their pointwise optimality condition is
\begin{equation}
    \beta\log
    \frac{\pi^\star_\beta(y\mid x)}{\pi_t(y\mid x)}
    =
    r(x,y)-V^\star_\beta(x)
    \equiv A^\star_\beta(x,y).
    \label{eq:mintrl-kl-condition}
\end{equation}
Thus, $A^\star_\beta(x,y)$ directly specifies how the likelihood of each trajectory should change relative to $\pi_t$.
Following squared advantage-regression methods
\citep{kimi2025k15,brantley2026accelerating,kimik2,ritter2026llms}, we fit this condition directly. The corresponding regression objective is
\begin{equation}
    \mathcal{L}(\theta)
    =
    \E_{(x,y)\in\D}
    \left[
        \beta\log
        \frac{\pi_\theta(y\mid x)}{\pi_t(y\mid x)}
        -
        A^\star_\beta(x,y)
    \right]^2 .
    \label{eq:mintrl-population-loss}
\end{equation}
\textbf{Crucially, the regression target in \eqref{eq:mintrl-population-loss} is well-defined for any trajectory $y$ with $\pi_t(y\mid x)>0$, regardless of the behavior policy that generated it.}
Thus, semi-on-policy trajectories can be used directly as regression examples without behavior-policy importance weighting; the sampling distribution only determines which trajectories are observed.

\subsection{Practical Implementation}
\label{sec:practical-implementation}

\paragraph{Step-level intervention localization.}
In practice, reliably identifying the exact first erroneous token is impossible.
We therefore perform intervention at the level of reasoning or code steps rather than individual tokens.
Specifically, each policy-generated chunk is first segmented into steps using double-newline delimiters (\texttt{\textbackslash n\textbackslash n}) or code-fence boundaries.
The judge then reviews the resulting sequence of steps and identifies the earliest incorrect step.
The judge and intervention prompts are provided in
Appendix~\ref{app:prompt-formats}.
We retain all steps preceding this point, discard the remaining suffix, and invoke $\pji$ to generate a short corrective continuation from the retained prefix.

\paragraph{Control-based advantage estimation.}
Computing the soft value $V^\star_\beta(x)$ in \eqref{eq:mintrl-soft-value} exactly requires marginalizing over all possible responses.
In practice, for each prompt $x$, we collect $N_m$ semi-on-policy rollouts using Alg~\ref{alg:intervened-rollout} together with $N_c$ pure on-policy control rollouts from the current policy $\pi_t$,
\begin{equation}
y_j^{\mathrm{ctrl}}
\sim
\pi_t(\cdot\mid x),
\qquad
j=1,\ldots,N_c.
\label{eq:mintrl-control-rollouts}
\end{equation}
The control rollouts are used to estimate the value baseline, while both control and semi-on-policy rollouts are included as regression examples.
To decouple value estimation from the strength of policy regularization, we use two temperatures: $\beta_1$ for the value target and $\beta_2$ for policy regression, following prior advantage-regression formulations~\citep{brantley2026accelerating,ritter2026llms}.
The resulting Monte Carlo estimate of the soft value is
\begin{equation}
\widehat V_{\beta_1}(x)
=
\beta_1
\log\left[
\frac{1}{N_c}
\sum_{j=1}^{N_c}
\exp\left(
r(x,y_j^{\mathrm{ctrl}})/\beta_1
\right)
\right].
\label{eq:mintrl-soft-value-estimator}
\end{equation}
The temperature $\beta_1$ controls how the value target aggregates rewards: as $\beta_1\rightarrow 0$, the soft value approaches the maximum observed reward, whereas as $\beta_1\rightarrow\infty$, it approaches the empirical mean reward.
We adopt the latter regime because the mean-reward baseline has been found effective in practice~\citep{kimi2025k15}, and choosing $\beta_1\gg\beta_2$ leads to a more conservative regression target that is less sensitive to a small number of high-reward off-policy trajectories~\citep{sakhi2026pessimistic}.
Accordingly, we take the large-$\beta_1$ limit and use
\begin{equation}
\widehat V_{\mathrm{ctrl}}(x)
=
\frac{1}{N_c}
\sum_{j=1}^{N_c}
r(x,y_j^{\mathrm{ctrl}}),
\qquad
\widehat A(x,y)
=
r(x,y)-\widehat V_{\mathrm{ctrl}}(x).
\label{eq:mintrl-control-advantage}
\end{equation}
Importantly, we estimate the baseline exclusively from on-policy control rollouts, while using both control and semi-on-policy rollouts as regression examples.
The same rule-based reward is applied to both groups.

\paragraph{A relaxed anchor for intervention tokens.}
The regression objective in \eqref{eq:mintrl-population-loss}
anchors the likelihood of each trajectory to the current policy $\pi_t$.
For intervention-authored tokens, however, $\pi_t$ may assign extremely
small probabilities, which can make the original reference overly
restrictive.
To relax this reference, we consider using their log probabilities
under the intervention policy $\pji$.
These log probabilities may be unavailable when $\pji$ is accessed
as a black-box model.
We therefore anchor intervention tokens with a constant $\kappa<0$:
\begin{equation}
\log\widetilde{\pi}_t(y_s\mid x,y_{<s})
=
\begin{cases}
\log\pi_t(y_s\mid x,y_{<s}),
& y_s \text{ is policy-authored},\\
\kappa,
& y_s \text{ is intervention-authored}.
\end{cases}
\label{eq:mintrl-intervention-anchor}
\end{equation}
Intuitively, $\kappa$ can be viewed as a shared reference level
motivated by the intervention policy's expected per-token log
probability,
$\mathbb{E}_{v}
[\log\pji(v\mid x,y_{<s})]$.
This interpretation is heuristic: $\kappa$ remains a tunable reference
value, and we do not assume that it accurately estimates this expectation.
The anchor modification enters the sequence-level reference score
additively in log space and is confined to sparse intervention spans,
while policy-authored tokens remain anchored to $\pi_t$.
Accordingly, we use it as a practical anchoring device rather than interpreting it as a standalone normalized policy.

\paragraph{Early stopping.}
In practice, the judge-intervention policy $\pji$ is itself imperfect, and the corrections it provides may occasionally be erroneous.
Nevertheless, during the early stage of training, these interventions can still provide useful guidance by helping the current policy reach successful trajectories that it would otherwise struggle to discover through on-policy sampling alone.
As the policy improves, the marginal benefit of intervention diminishes, while erroneous or unnecessary corrections continue to introduce harmful off-policy information.
We therefore use intervention only during an initial phase of training.
After a preset intervention phase, we disable $\pji$ and continue training with standard on-policy RL. This concentrates external guidance in the early stage, where it is most useful, while keeping the later stage fully on-policy.

\section{Theoretical Perspective}
\label{sec:coverage-learnability}

We formalize \mintrl as a sparse perturbation of the current policy: its rollout
distribution $\mu_N$ follows $\pi_t$ except at most $N$ corrected token
positions.  
This turns our problem into analyzing how sparse corrections
expand the support of correct rollouts.  Under a finite rollout budget,
on-policy training behaves as though the population reward were truncated to
empirically visible support.  
We first characterize this objective gap and
then bound how far \mintrl can expand support while remaining learnable.  This
perspective is related to empirical-support analyses of RLVR and classical
coverage conditions in reinforcement learning
\citep{wu2025invisible,xie2022coverage}.

\begin{definition}[$\epsilon$-support of correct rollouts]
\label{definition:epsilon-support}
For the binary verifier reward considered in this section, let
$\mathcal{C}(x)=\{y\in\mathcal{Y}:r(x,y)=1\}$ denote the set of correct
rollouts for prompt $x$.  For any threshold $\epsilon\in(0,1)$, the
$\epsilon$-support of correct rollouts under policy $\pi$ is
\begin{equation}
    \operatorname{supp}_{\epsilon}(\pi;x)
    =
    \left\{
        y\in\mathcal{C}(x):
        \pi(y\mid x)>\epsilon
    \right\}.
    \label{eq:epsilon-support}
\end{equation}
\end{definition}

Finite-budget visibility and trainability need not share the same likelihood
threshold.  We capture this distinction as follows.

\begin{assumption}[Inference and training support]
\label{assumption:inference-training-support}
There exist thresholds
$\epsilon_{\mathrm{train}}<\epsilon_{\mathrm{inf}}$ such that
$\operatorname{supp}_{\epsilon_{\mathrm{inf}}}(\pi_t;x)$ denotes the rollouts
that can be empirically sampled under finite-budget inference, while
$\operatorname{supp}_{\epsilon_{\mathrm{train}}}(\pi_t;x)$ denotes the
rollouts from which RL training can still learn.
\end{assumption}

\begin{wrapfigure}{r}{0.43\linewidth}
    \centering
    \vspace{-1em}
    \includegraphics[
        width=\linewidth
    ]{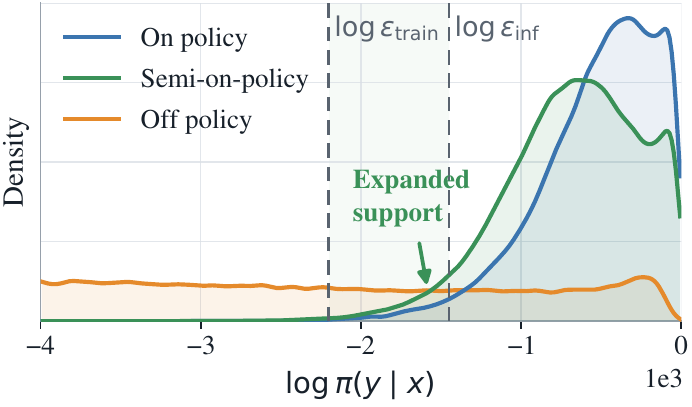}
    \caption{\textbf{Empirical support expansion.}
    The dashed thresholds visualize
    Assumption~\ref{assumption:inference-training-support}.  \mintrl adds mass
    beyond inference support while remaining within training support,
    illustrating Theorem~\ref{theorem:sparse-correction-support};
    off-policy rollouts extend largely beyond the trainable region.}
    \label{fig:rollout-logprob-distribution}
\end{wrapfigure}

The ordering
$\epsilon_{\mathrm{train}}<\epsilon_{\mathrm{inf}}$ is reasonable:
finite-budget sampling only reveals an empirically visible subset of correct
rollouts \citep{wu2025invisible,chen2026does}, while asynchronous and
stale-rollout studies show that data can remain useful for training after
becoming off-policy
\citep{hou2026single,zheng2026prosperity}.
We first characterize the optimization limit imposed by the smaller
inference support.

\begin{lemma}[Finite-rollout objective gap]
\label{lemma:finite-budget-objective}
Under Assumption~\ref{assumption:inference-training-support}, the empirical RL
objective at iteration $t$ is
\begin{equation}
    \E_{y\sim\pi(\cdot\mid x)}
    \left[
        r(x,y)\,
        \ind\!\left\{
            y\in\operatorname{supp}_{\epsilon_{\mathrm{inf}}}(\pi_t;x)
        \right\}
    \right].
    \label{eq:visible-objective}
\end{equation}
In this regard, the objective differs from the population objective
$\E_{y\sim\pi(\cdot\mid x)}[r(x,y)]$: the optimality gap is one when the
inference support is empty and zero otherwise.
\end{lemma}

The proof is provided in Appendix~\ref{app:proof-finite-budget-objective}.
Lemma~\ref{lemma:finite-budget-objective} shows why a larger inference support
helps: it makes more correct rollouts visible under finite-budget sampling,
reducing the gap between the empirical and population objectives.  We next
quantify how far sparse optimal corrections can expand this support under a
mild token-level reachability assumption.

\begin{assumption}[Optimal-token reachability]
\label{assumption:optimal-token-support}
There exists $\nu>0$ such that
$\pi(y_t^\star\mid x,y_{<t})\geq\nu$ for every policy $\pi\in\Pi$, prompt
$x$, prefix $y_{<t}$, and optimal next token $y_t^\star$.
\end{assumption}

This is a simplifying reachability assumption used to obtain a uniform support envelope.
Under this assumption, the support expansion of \mintrl is characterized by
the following envelope.

\begin{theorem}[Support envelope of sparse optimal corrections]
\label{theorem:sparse-correction-support}
Fix the current policy $\pi_t$ and a prompt $x$.  Let $\mu_N$ follow $\pi_t$
except at a set of at most $N$ positions, where it emits an optimal next
token $y_t^\star$.  The visible correct support of the combined
current-policy and corrected rollout pool is
\begin{equation}
    \mathcal{S}_{\epsilon_{\mathrm{inf}},N}(\pi_t,\mu_N;x)
    =
    \operatorname{supp}_{\epsilon_{\mathrm{inf}}}(\pi_t;x)
    \cup
    \operatorname{supp}_{\epsilon_{\mathrm{inf}}}(\mu_N;x).
    \label{eq:mixed-visible-support}
\end{equation}
Under Assumption~\ref{assumption:optimal-token-support},
\begin{equation}
    \operatorname{supp}_{\epsilon_{\mathrm{inf}}}(\pi_t;x)
    \subseteq
    \mathcal{S}_{\epsilon_{\mathrm{inf}},N}(\pi_t,\mu_N;x)
    \subseteq
    \operatorname{supp}_{\epsilon_{\mathrm{inf}}\nu^N}(\pi_t;x).
    \label{eq:sparse-correction-support-envelope}
\end{equation}
\end{theorem}

The proof is provided in Appendix~\ref{app:proof-sparse-correction-support}.

The left inclusion preserves the current policy's visible support, while the
right inclusion shows that $N$ corrections can expose rollouts up to a factor
of $(1/\nu)^N$ less likely under $\pi_t$.  This expansion can provide the
positive signal missing in Lemma~\ref{lemma:finite-budget-objective}.
However, the theorem guarantees that corrected rollouts remain within
training support only while
$\epsilon_{\mathrm{inf}}\nu^N\geq\epsilon_{\mathrm{train}}$.  Beyond this
range, the expanded support may include rollouts from which the policy can no
longer reliably learn.  Thus, $N$ trades broader coverage against the
learnability boundary set by $\epsilon_{\mathrm{train}}$.

\section{Experiments}
\label{sec:experiments}
\paragraph{Experimental setup.}
In the main experiments, we use Qwen3-1.7B and Qwen3-4B, both operated
in non-thinking mode, as the policy models \citep{yang2025qwen3}, and
Qwen3-4B-Instruct-2507 as the judge--intervention model.  
We study both math reasoning and code generation.  
For math, we use a subset of the AceReason-Nemotron training prompts \citep{chen2026acereason}, filtering out overly easy problems to avoid spending rollout compute on zero advantage examples.  
For code, we use the DeepCoder-Preview training set \citep{luo2025deepcoder}.  

The semi-on-policy rollout generation is controlled by three hyperparameters: 1). \emph{chunk size}, which determines how many tokens $\pi_t$ generates between reviews; 2). \emph{number of judge reviews}, which limits how many times $\pi_{\mathrm{JI}}$ can inspect a trajectory; and 3). \emph{intervention continuation length}, which bounds the number of tokens generated by $\pi_{\mathrm{JI}}$ at each intervention.
We use $(512,4,64)$ for math and $(128,4,64)$ for code, respectively.
We implement two variants of the algorithm: MInTRL-Proxy retains the current-policy log probabilities as anchors for intervention-authored tokens, whereas MInTRL-Const replaces them with the constant anchor in \eqref{eq:mintrl-intervention-anchor}.
%For the relaxed anchor constant in \eqref{eq:mintrl-intervention-anchor}, we set $\kappa = -0.2$ for math and $\kappa = -0.1$ for code.
%More hyperparameter settings are provided in
%Appendix~\ref{app:experimental-setups}.

We evaluate math reasoning on
AIME 2025 \citep{balunovic2026matharena}, AIME 2026
\citep{dekoninck2026beyond}, and HMMT February 2025
\citep{balunovic2026matharena}, and code generation on LiveCodeBench
\citep{jain2025livecodebench}, HumanEval+ \citep{liu2023your}, and MBPP+
\citep{liu2023your}.  
%At evaluation time, we sample with temperature $1.0$, top-$p$ of $1.0$, and a maximum generation length of $32{,}768$ tokens.
For every problem, we draw 32 samples and report their mean correctness, which estimates Pass@1.
More hyperparameter settings and implementation details are provided in
Appendix~\ref{app:experimental-setups}.

\paragraph{Baselines.}
We compare \mintrl against four categories of baselines:
\begin{enumerate}
    \item \textbf{Reference baselines.} The unmodified base policy and standard
    on-policy GRPO \citep{shao2024deepseekmath}, which represent no post-training and
    conventional on-policy RL, respectively.
    \item \textbf{Off-policy baselines.} SFT+GRPO first fine-tunes the policy
    on complete trajectories sampled from $\pji$, and then continues training
    with on-policy GRPO.
    \item \textbf{On-policy baselines.} Standard OPD
    \citep{agarwal2024onpolicy,lu2025onpolicydistillation} performs token-level
    distillation on trajectories sampled from the policy itself.
    \item \textbf{Intervention-based baseline.} MENTOR
    \citep{jiang2026selective} selectively injects expert guidance into rollout
    generation and optimizes the resulting trajectories with a GRPO-style
    objective, making it the closest existing intervention-based RL method to
    \mintrl.
\end{enumerate}
Together, these baselines provide broad coverage of representative on-policy and off-policy post-training methods, enabling a comprehensive comparison with \mintrl.

\subsection{Overall results}
As shown in Table~\ref{tab:main-results}, \mintrl achieves the best overall performance across both model scales and domains, outperforming standard on-policy RL, distillation-based methods, and teacher-guided baselines.
For Qwen3-1.7B, MInTRL-Const reaches average scores of $35.45$ on mathematics and $61.95$ on code.
Compared with the stronger of GRPO and OPD, this corresponds to gains of $+13.61$ and $+14.12$ percentage points (pp), respectively, with improvements on every individual benchmark.
The gains remain substantial at the 4B scale: MInTRL-Const achieves $55.73$ on mathematics and $72.63$ on code, improving over the two baselines by $+3.02$ pp and $+6.80$ pp, respectively.
Moreover, MInTRL-Const generally outperforms the theoretically motivated MInTRL-Proxy, achieving higher aggregate performance across both domains and model scales.
We hypothesize that this is because, under relatively weak policies, useful intervention tokens can have extremely low probability under the current policy and are therefore overly constrained by the policy-dependent anchor used in MInTRL-Proxy.
Overall, these results demonstrate that sparse local interventions can effectively leverage off-policy guidance while preserving the benefits of predominantly on-policy rollout generation.

\begin{table}[t]
    \centering
    \caption{Performance on math and code benchmarks.\tablefootnote{All values come from the checkpoint with the highest three-benchmark average. Final-checkpoint results are reported in Appendix~\ref{app:fixed-checkpoint-comparison}.}\tablefootnote{A recent method combining intervention with OPD \citep{xu2026relayopd} reports scores of 32.81 and 30.52 on AIME25 and AIME26, respectively. We did not reproduce these results locally; both reported scores are below those of our algorithm.}}
    \label{tab:main-results}
    \resizebox{\textwidth}{!}{%
    \begin{tabular}{lrrrrrrrr}
        \toprule
        & \multicolumn{4}{c}{Mathematics} & \multicolumn{4}{c}{Code} \\
        \cmidrule(lr){2-5}\cmidrule(l){6-9}
        Method & AIME25 & AIME26 & HMMT25 & Avg. & LCB & HE+ & MBPP+ & Avg. \\
        \midrule
        \multicolumn{9}{l}{\textit{Qwen3-4B-Instruct-2507}} \\ 
        & 47.50 & 53.02 & 28.65 & 43.06 & 35.65 & 82.34 & 71.58 & 63.19 \\
        \midrule
        \multicolumn{9}{l}{\textit{Qwen3-1.7B}} \\
        Base        & 8.85 & 8.85 & 4.48 & 7.40 & 12.99 & 55.89 & 52.18 & 40.35 \\
        GRPO        & 22.08 & 16.46 & 14.69 & 17.74 & 21.74 & 64.63 & 54.65 & 47.01 \\
        OPD         & 25.73 & 24.27 & 15.52 & 21.84 & 24.13 & 63.66 & 55.70 & 47.83 \\
        MENTOR      & 31.04 & 28.75 & 18.23 & 26.01 & 33.46 & 62.80 & 56.02 & 50.76 \\
        SFT+GRPO    & 31.56 & 27.60 & 16.46 & 25.21 & \underline{34.20} & \underline{78.28} & \underline{67.15} & \underline{59.88} \\
        MInTRL-Proxy & \underline{36.77} & \underline{32.50} & \underline{19.58} & \underline{29.62} & 32.84 & 73.95 & 60.09 & 55.63 \\
        MInTRL-Const & \textbf{40.73} & \textbf{41.56} & \textbf{24.06} & \textbf{35.45} & \textbf{37.85} & \textbf{80.81} & \textbf{67.20} & \textbf{61.95} \\
        \midrule
        \multicolumn{9}{l}{\textit{Qwen3-4B}} \\
        Base        & 17.40 & 15.10 & 11.77 & 14.76 & 23.34 & 72.71 & 61.40 & 52.48 \\
        GRPO        & 60.83 & 60.31 & 36.98 & 52.71 & 50.06 & 80.32 & 67.13 & 65.83 \\
        OPD         & 43.02 & 50.10 & 28.44 & 40.52 & 35.11 & 77.31 & 63.98 & 58.80 \\
        MENTOR      & \underline{61.88} & \textbf{65.00} & \textbf{39.17} & \underline{55.35} & \underline{53.20} & 86.38 & 75.83 & 71.80 \\
        SFT+GRPO    & 57.60 & 63.65 & 34.27 & 51.84 & 49.25 & \underline{87.21} & \underline{75.98} & 70.81 \\
        MInTRL-Proxy & 57.60 & 58.23 & \underline{37.71} & 51.18 & 52.68 & 86.79 & \textbf{76.20} & \underline{71.89} \\
        MInTRL-Const & \textbf{65.31} & \underline{64.17} & \underline{37.71} & \textbf{55.73} & \textbf{54.19} & \textbf{87.84} & 75.85 & \textbf{72.63} \\
        \bottomrule
    \end{tabular}
    }
\end{table}

\begin{comment}
            \addlinespace[1pt]
        {\itshape\color{black!60} $\Delta$ vs GRPO/OPD}
        & \gain{15.00} & \gain{17.29} & \gain{8.54} & \gain{13.61}
        & \gain{13.72} & \gain{16.18} & \gain{11.50} & \gain{14.12} \\

                        \addlinespace[1pt]
        {\itshape\color{black!60} $\Delta$ vs GRPO/OPD}
        & \gain{4.48} & \gain{3.86} & \gain{0.73} & \gain{3.02}
        & \gain{4.13} & \gain{7.52} & \gain{8.72} & \gain{6.80} \\
\end{comment}

\begin{figure*}[!t]
    \centering
    \includegraphics[width=0.76\textwidth]{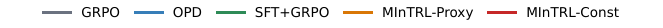}\\[-0.5em]
    \begin{minipage}[t]{0.326\textwidth}
        \centering
        \includegraphics[width=\linewidth]{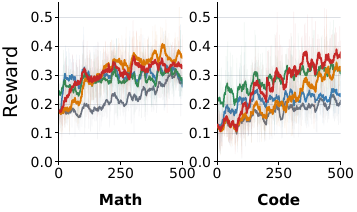}
    \end{minipage}\hfill%
    \begin{minipage}[t]{0.326\textwidth}
        \centering
        \includegraphics[width=\linewidth]{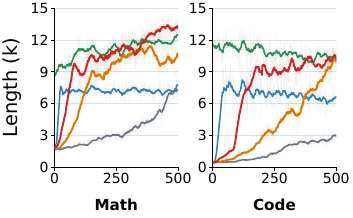}
    \end{minipage}\hfill%
    \begin{minipage}[t]{0.326\textwidth}
        \centering
        \includegraphics[width=\linewidth]{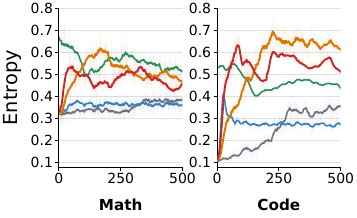}
    \end{minipage}
    \par\vspace{0.15em}
    \begin{minipage}[t]{0.326\textwidth}
        \centering
        {\small\textbf{(a)} Training reward.\par}
    \end{minipage}\hfill%
    \begin{minipage}[t]{0.326\textwidth}
        \centering
        {\small\textbf{(b)} Response length.\par}
    \end{minipage}\hfill%
    \begin{minipage}[t]{0.326\textwidth}
        \centering
        {\small\textbf{(c)} Entropy.\par}
    \end{minipage}
    \caption{\textbf{Training dynamics.}
    \mintrl improves reward while maintaining longer responses and higher empirical token entropy throughout training.}
    \label{fig:training-dynamics}
    \par\vspace{0.25em}
    \begin{minipage}[t]{0.49\textwidth}
        \centering
        \includegraphics[width=\linewidth]{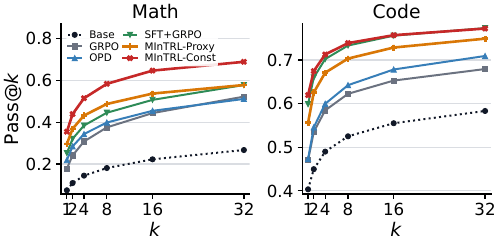}
    \end{minipage}\hfill%
    \begin{minipage}[t]{0.49\textwidth}
        \centering
        \includegraphics[width=\linewidth]{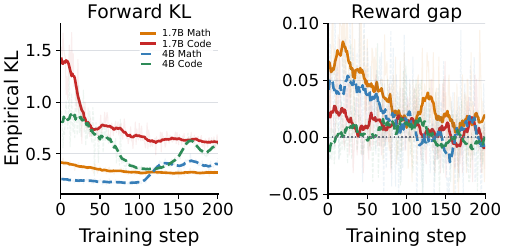}
    \end{minipage}
    \par\vspace{0.08em}
    \begin{minipage}[t]{0.49\textwidth}
        \centering
        {\small\textbf{(a)} Pass@$k$ performance.\par}
    \end{minipage}\hfill%
    \begin{minipage}[t]{0.49\textwidth}
        \centering
        {\small\textbf{(b)} Acquisition of intervention knowledge.\par}
    \end{minipage}
    \caption{\textbf{Inference-time scaling and intervention knowledge.}
    Panel (a) reports domain-averaged Pass@$k$ on evaluation benchmarks. Panel (b) shows the empirical
    forward KL on intervention tokens and the reward gap between semi-on-policy rollouts and on-policy rollouts}
    \label{fig:pass-at-k}
    \label{fig:intervention-knowledge}
\end{figure*}

\subsection{Training Dynamics}
Having established the overall performance gains of \mintrl, we next investigate how these gains develop throughout training.
We analyze the effects of \mintrl on policy exploration and diversity, dependence on intervention, and observable reasoning behavior.
Unless otherwise noted, the following training-dynamics analyses are based on the 1.7B experiments reported above.

\paragraph{MInTRL promotes exploration and policy diversity.}
As shown in Figure~\ref{fig:training-dynamics}, \mintrl achieves higher training reward while maintaining substantially higher policy entropy than GRPO and OPD.
It also ultimately surpasses SFT+GRPO in training reward, despite relying only on sparse local interventions rather than full teacher-generated trajectories.
Figure~\ref{fig:intervention-knowledge}(a) further shows that both \mintrl variants outperform GRPO in Pass@$k$ across the evaluated sampling budgets in both domains, with gains persisting at $k=32$.
Together, these results suggest that \mintrl improves access to successful reasoning paths while preserving policy diversity.
Local corrections introduce promising prefixes that are difficult for the current policy to reach through on-policy sampling alone.
By returning control to the policy after each correction, \mintrl allows it to explore its own continuations from these prefixes, combining external guidance with continued policy exploration.

\paragraph{MInTRL progressively reduces reliance on intervention.}
Figure~\ref{fig:intervention-knowledge}(b) provides two complementary signals that the policy gradually learns from $\pji$ during training.
First, the empirical forward KL measured on intervention tokens generally decreases from its initial level, suggesting closer alignment between the policy and $\pji$ on these tokens.
Second, the reward gap between semi-on-policy and purely on-policy rollouts progressively narrows toward zero.
Interventions provide a clear reward improvement early in training, but offer diminishing additional benefits as training progresses.
These trends suggest that the policy increasingly incorporates useful intervention guidance into its own generation.

\paragraph{MInTRL changes observable reasoning behavior.}
\begin{wrapfigure}{r}{0.48\linewidth}
    \centering
    \vspace{-1.0em}
    \includegraphics[width=\linewidth]{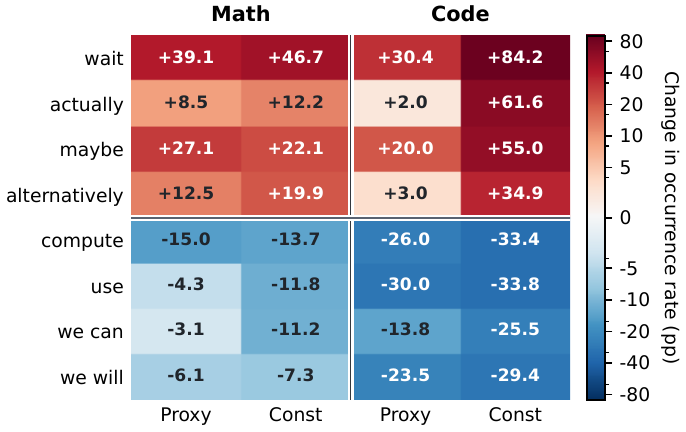}
    \caption{\textbf{Keyword shifts relative to GRPO.}
    Reconsideration and branching markers (\emph{wait}, \emph{actually},
    \emph{maybe}, and \emph{alternatively}) increase, whereas early execution
    and commitment markers (\emph{compute}, \emph{use}, \emph{we can}, and
    \emph{we will}) decrease.}
    \label{fig:reasoning-behavior-keywords}
    \vspace{-0.8em}
\end{wrapfigure}

We next examine whether MInTRL affects not only solution quality but also observable reasoning behavior.
We compare evaluation rollouts from GRPO and \mintrl variants across both math and code.
For a fair comparison, we restrict the analysis to the first 512 narrative words of each response, and measure how frequently a set of reflection- and reasoning-related keywords appears within this fixed window.

Figure~\ref{fig:reasoning-behavior-keywords} reveals a consistent shift in observable reasoning behavior across both domains.
Relative to GRPO, MInTRL increases the occurrence of reconsideration and branching keywords, while reducing execution- and commitment-oriented keywords.
For example, under MInTRL-Const on code, occurrences of \emph{wait}, \emph{actually}, and \emph{maybe} increase by 55.0--84.2 pp relative to GRPO, while those of the execution-oriented keywords \emph{compute}, \emph{use}, \emph{we can}, and \emph{we will} decrease by 25.5--33.8 pp.
MInTRL-Proxy exhibits the same directional pattern, although with smaller changes in keyword frequency.
This smaller shift is consistent with MInTRL-Proxy's stronger policy anchor, as discussed in Section~\ref{sec:practical-implementation}.
Several qualitative examples in Appendix~\ref{app:qualitative-intervention-examples} further illustrate how interventions introduce explicit reflection cues into rollouts.
Together, these shifts suggest that MInTRL induces a more reflective and extended observable reasoning pattern, with more frequent reconsideration before committing to a particular execution path.

\section{Ablation}

\subsection{MInTRL with Self Judge-Intervention}
\label{sec:self-intervention}

\begin{wraptable}{r}{0.52\linewidth}
    \centering
    \vspace{-1.0em}
    \caption{
    {Performance with self judge-intervention.}}
    \label{tab:self-judge-pass1}
    \footnotesize
    \setlength{\tabcolsep}{2.8pt}
    \renewcommand{\arraystretch}{1.05}
    \begin{tabular}{lrrrr}
        \toprule
        Method & AIME25 & AIME26 & HMMT25 & Avg. \\
        \midrule
        GRPO        & 50.63 & 58.13 & \underline{33.02} & 47.26 \\
        OPSD        & 47.71 & 55.73 & 31.67 & 45.03 \\
        MInTRL-Proxy & \textbf{59.06} & \textbf{61.56} & \textbf{36.25} & \textbf{52.29} \\
        MInTRL-Const & \underline{55.10} & \underline{58.75} & 31.56 & \underline{48.47} \\
        \bottomrule
    \end{tabular}
    \vspace{-0.8em}
\end{wraptable}

In the above section, the judge-intervention policy is instantiated by a model that is stronger than the policy being trained.
We next examine whether this assumption is necessary.
Instead of using a separate stronger teacher, we construct $\pi_{\mathrm{JI}}$ from the base policy $\pi_0$ itself, while providing it with additional privileged context.
Following recent context-conditioned on-policy self-distillation (OPSD) approaches
\citep{zhao2026self,shenfeld2026self}, for each $x$, we construct an auxiliary context $z$ that provides additional information unavailable to the policy during normal generation.
We then define the judge-intervention policy as $\pi_{\mathrm{JI}}(\cdot\mid x)=\pi_0(\cdot\mid x,z)$, while the policy being trained only conditions on the original input $x$.
In this ablation, we use Qwen3-4B-Instruct-2507 for both roles.
The results are shown in Table~\ref{tab:self-judge-pass1}.
Both MInTRL variants outperform the corresponding baselines on average.
Interestingly, MInTRL-Proxy outperforms MInTRL-Const in this setting, achieving an average score of $52.29$ versus $48.47$, in contrast to the trend observed in Table~\ref{tab:main-results}.
We hypothesize that this is because Qwen3-4B-Instruct-2507 provides a substantially stronger base policy, which already assigns reasonable probability to many intervention-authored tokens.
Consequently, relaxing their reference probabilities with the constant anchor becomes less necessary, and retaining the original policy probabilities provides a more stable regularization signal.

\FloatBarrier

\subsection{Off-policy intensity}
\label{sec:intervention-intensity}

\begin{wrapfigure}{r}{0.56\linewidth}
    \centering
    \vspace{-1.0em}
    \includegraphics[
        width=\linewidth
    ]{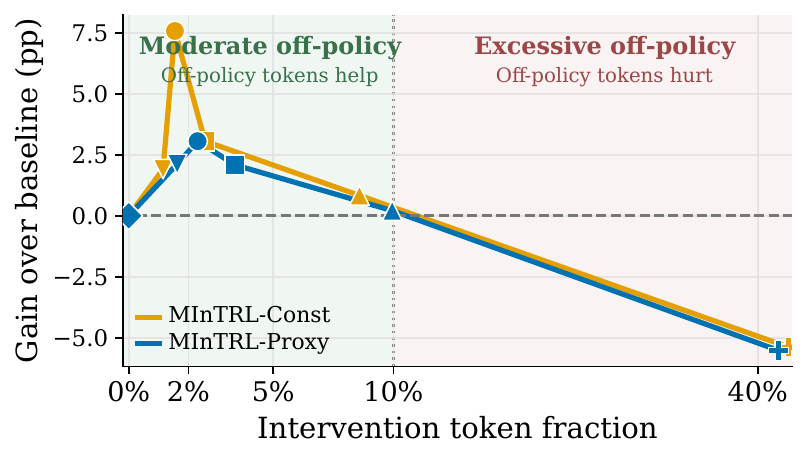}
    \caption{\textbf{Performance peaks at moderate off-policy intensity.}
    Moderate off-policy guidance improves performance, whereas excessive
    off-policy content erodes the gain and can even underperform the baseline.}
    \label{fig:intervention-intensity}
    \vspace{-0.8em}
\end{wrapfigure}

We next study how the amount of off-policy intervention affects RL performance.
We vary the intervention intensity by changing the maximum number of judge reviews and the maximum intervention continuation length, which together control the fraction of rollout tokens authored by $\pi_{\mathrm{JI}}$.
Figure~\ref{fig:intervention-intensity} plots the Pass@1 improvement over the no-intervention baseline\footnote{The 0\% intervention condition is a pure on-policy control trained with the same advantage-regression objective and training setup, with intervention disabled.} against the resulting off-policy token fraction, averaged over training.
Both MInTRL variants exhibit a clear non-monotonic trend.
Performance initially improves as a small amount of off-policy guidance is introduced, but degrades once interventions become too frequent or too long.
The best performance is achieved at a moderate off-policy fraction of roughly 2--4\%, whereas excessive intervention eventually underperforms the purely on-policy baseline.
This inverted-U trend is consistent with the coverage--learnability trade-off analyzed in Section~\ref{sec:coverage-learnability}: moderate intervention expands rollout coverage, whereas excessive off-policy content makes the resulting trajectories harder for the current policy to learn from.

\FloatBarrier

\subsection{Robustness to the Judge--Intervention Model}
\label{sec:judge-intervention-robustness}

Our main experiments instantiate the judge--intervention policy with Qwen3-4B-Instruct-2507. 
To assess whether \mintrl depends on this particular model, we replace it with DeepSeek-V4-Flash, a strong model from a different model family.
As shown in Appendix~\ref{app:judge-intervention-robustness}, \mintrl remains effective with DeepSeek-V4-Flash as the judge--intervention model, consistently outperforming the GRPO and OPD baselines across policy scales and task domains.
Interestingly, using a stronger standalone judge does not consistently yield better downstream performance. 
Although DeepSeek-V4-Flash improves some configurations, the gains vary across model scales, domains, and anchoring variants. 
We hypothesize that stronger corrections may also induce greater policy mismatch: compared with Qwen3-4B-Instruct-2507, DeepSeek-V4-Flash may generate intervention tokens that are less likely under the Qwen policy, making these corrections more difficult to learn from despite the judge's higher standalone quality.

\section{Conclusion}
\label{sec:conclusion}

We introduced \mintrl, which augments otherwise on-policy rollouts with sparse, local interventions and learns from the resulting trajectories using a regression-based RL objective.
Our theoretical and empirical results reveal a coverage--learnability trade-off: moderate intervention exposes useful reasoning paths beyond ordinary on-policy sampling, whereas excessive off-policy content makes the resulting trajectories harder to learn from.
Across mathematical reasoning and code generation, \mintrl consistently improves over on-policy RL and distillation baselines, suggesting that carefully controlled off-policy experience can expand exploration without sacrificing policy proximity.

\bibliography{iclr2027_conference}

@inproceedings{agarwal2024onpolicy,
  title={On-policy distillation of language models: Learning from self-generated mistakes},
  author={Agarwal, Rishabh and Vieillard, Nino and Zhou, Yongchao and Stanczyk, Piotr and Ramos Garea, Sabela and Geist, Matthieu and Bachem, Olivier},
  booktitle={International Conference on Learning Representations},
  volume={2024},
  pages={21246--21263},
  year={2024}
}

@inproceedings{jiang2026selective,
  title={Selective expert guidance for effective and diverse exploration in reinforcement learning of llms},
  author={Jiang, Zishang and Han, Jinyi and Wang, Xinyi and Jiang, Sihang and Dai, Zhaoqian and Shuguang, Ma and Yu, Fei and Liang, Jiaqing and Xiao, Yanghua and others},
  booktitle={International Conference on Learning Representations},
  volume={2026},
  pages={62980--63006},
  year={2026}
}

@article{lu2025onpolicydistillation,
  author = {Kevin Lu and Thinking Machines Lab},
  title = {On-Policy Distillation},
  journal = {Thinking Machines Lab: Connectionism},
  year = {2025},
  note = {https://thinkingmachines.ai/blog/on-policy-distillation},
  doi = {10.64434/tml.20251026},
}

@article{zhao2026self,
  title={Self-distilled reasoner: On-policy self-distillation for large language models},
  author={Zhao, Siyan and Xie, Zhihui and Liu, Mengchen and Huang, Jing and Pang, Guan and Chen, Feiyu and Grover, Aditya},
  journal={arXiv preprint arXiv:2601.18734},
  year={2026}
}

@article{shenfeld2026self,
  title={Self-distillation enables continual learning},
  author={Shenfeld, Idan and Damani, Mehul and H{\"u}botter, Jonas and Agrawal, Pulkit},
  journal={arXiv preprint arXiv:2601.19897},
  year={2026}
}

@article{balunovic2026matharena,
  title={Matharena: Evaluating llms on uncontaminated math competitions},
  author={Balunovic, Mislav and Dekoninck, Jasper and Petrov, Ivo and Jovanovi{\'c}, Nikola and Vechev, Martin},
  journal={Advances in Neural Information Processing Systems},
  volume={38},
  year={2026}
}

@article{chen2026acereason,
  title={Acereason-nemotron: Advancing math and code reasoning through reinforcement learning},
  author={Chen, Yang and Yang, Zhuolin and Liu, Zihan and Lee, Chankyu and Xu, Peng and Shoeybi, Mohammad and Catanzaro, Bryan and Ping, Wei},
  journal={Advances in neural information processing systems},
  volume={38},
  pages={110320--110345},
  year={2026}
}

@article{dekoninck2026beyond,
  title={Beyond benchmarks: Matharena as an evaluation platform for mathematics with llms},
  author={Dekoninck, Jasper and Jovanovi{\'c}, Nikola and Gehrunger, Tim and R{\"o}gnvaldsson, K{\'a}ri and Petrov, Ivo and Sun, Chenhao and Vechev, Martin},
  journal={arXiv preprint arXiv:2605.00674},
  year={2026}
}

@inproceedings{jain2025livecodebench,
  title={Livecodebench: Holistic and contamination free evaluation of large language models for code},
  author={Jain, Naman and Gu, Alex and Li, Wen-Ding and Yan, Fanjia and Zhang, Tianjun and Wang, Sida and Solar-Lezama, Armando and Sen, Koushik and Stoica, Ion},
  booktitle={International Conference on Learning Representations},
  volume={2025},
  pages={58791--58831},
  year={2025}
}

@article{liu2023your,
  title={Is your code generated by chatgpt really correct? rigorous evaluation of large language models for code generation},
  author={Liu, Jiawei and Xia, Chunqiu Steven and Wang, Yuyao and Zhang, Lingming},
  journal={Advances in neural information processing systems},
  volume={36},
  pages={21558--21572},
  year={2023}
}

@article{luo2025deepcoder,
  title={Deepcoder: A fully open-source 14b coder at o3-mini level},
  author={Luo, Michael and Tan, Sijun and Huang, Roy and Patel, Ameen and Ariyak, Alpay and Wu, Qingyang and Shi, Xiaoxiang and Xin, Rachel and Cai, Colin and Weber, Maurice and others},
  journal={Notion Blog},
  volume={1},
  year={2025}
}

@article{yang2025qwen3,
  title={Qwen3 technical report},
  author={Yang, An and Li, Anfeng and Yang, Baosong and Zhang, Beichen and Hui, Binyuan and Zheng, Bo and Yu, Bowen and Gao, Chang and Huang, Chengen and Lv, Chenxu and others},
  journal={arXiv preprint arXiv:2505.09388},
  year={2025}
}

@article{brantley2026accelerating,
  title={Accelerating rl for llm reasoning with optimal advantage regression},
  author={Brantley, Kiant{\'e} and Chen, Mingyu and Gao, Zhaolin and Lee, Jason and Sun, Wen and Zhan, Wenhao and Zhang, Xuezhou},
  journal={Advances in Neural Information Processing Systems},
  volume={38},
  pages={151492--151531},
  year={2026}
}

@article{hinton2015distilling,
  title={Distilling the knowledge in a neural network},
  author={Hinton, Geoffrey and Vinyals, Oriol and Dean, Jeff},
  journal={arXiv preprint arXiv:1503.02531},
  year={2015}
}

@article{guo2025deepseek,
  title={Deepseek-r1: Incentivizing reasoning capability in llms via reinforcement learning},
  author={Guo, Daya and Yang, Dejian and Zhang, Haowei and Song, Junxiao and Wang, Peiyi and Zhu, Qihao and Xu, Runxin and Zhang, Ruoyu and Ma, Shirong and Bi, Xiao and others},
  journal={arXiv preprint arXiv:2501.12948},
  year={2025}
}

@article{kimi2025k15,
  title={Kimi k1. 5: Scaling reinforcement learning with llms},
  author={Team, Kimi and Du, Angang and Gao, Bofei and Xing, Bowei and Jiang, Changjiu and Chen, Cheng and Li, Cheng and Xiao, Chenjun and Du, Chenzhuang and Liao, Chonghua and others},
  journal={arXiv preprint arXiv:2501.12599},
  year={2025}
}

@article{kimik2,
  title={Kimi k2: Open agentic intelligence},
  author={Team, Kimi and Bai, Yifan and Bao, Yiping and Charles, Y and Chen, Cheng and Chen, Guanduo and Chen, Haiting and Chen, Huarong and Chen, Jiahao and Chen, Ningxin and others},
  journal={arXiv preprint arXiv:2507.20534},
  year={2025}
}

@article{hou2026single,
  title={Single-Rollout Asynchronous Optimization for Agentic Reinforcement Learning},
  author={Hou, Zhenyu and Li, Yujiang and Tang, Jie and Dong, Yuxiao},
  journal={arXiv preprint arXiv:2607.07508},
  year={2026}
}

@article{ritter2026llms,
  title={Llms can learn to reason via off-policy rl},
  author={Ritter, Daniel and Oertell, Owen and Guo, Bradley and Chang, Jonathan and Brantley, Kiant{\'e} and Sun, Wen},
  journal={arXiv preprint arXiv:2602.19362},
  year={2026}
}

@article{sakhi2026pessimistic,
  title={Off-Policy Learning to Reason Works Because It Is More Pessimistic Than You Think},
  author={Sakhi, Otmane and Arzhantsev, Aleksei and Aouali, Imad and Vasile, Flavian},
  journal={arXiv preprint arXiv:2605.28150},
  year={2026}
}

@article{schulman2017ppo,
  title={Proximal policy optimization algorithms},
  author={Schulman, John and Wolski, Filip and Dhariwal, Prafulla and Radford, Alec and Klimov, Oleg},
  journal={arXiv preprint arXiv:1707.06347},
  year={2017}
}

@article{shao2024deepseekmath,
  title={Deepseekmath: Pushing the limits of mathematical reasoning in open language models},
  author={Shao, Zhihong and Wang, Peiyi and Zhu, Qihao and Xu, Runxin and Song, Junxiao and Bi, Xiao and Zhang, Haowei and Zhang, Mingchuan and Li, YK and Wu, Yang and others},
  journal={arXiv preprint arXiv:2402.03300},
  year={2024}
}

@inproceedings{zheng2026prosperity,
  title={Prosperity before Collapse: How Far Can Off-Policy RL Reach with Stale Data on LLMs?},
  author={Zheng, Haizhong and Zhao, Jiawei and Chen, Beidi},
  booktitle={International Conference on Learning Representations},
  volume={2026},
  pages={82657--82679},
  year={2026}
}

@article{xu2026relayopd,
  title   = {Pass the Baton: Trajectory-Relayed On-Policy Distillation},
  author  = {Xu, Haolei and Xu, Xiaowen and Hong, Haiwen and Ni, Zixuan and Li, Hongxing and Qiu, Yiwen and Lu, Weiming and Shen, Yongliang},
  journal = {arXiv preprint arXiv:2607.26057},
  year    = {2026}
}

@article{wu2025invisible,
  title   = {The Invisible Leash? Why {RLVR} May or May Not Escape Its Origin},
  author  = {Wu, Fang and Xuan, Weihao and Lu, Ximing and Liu, Mingjie and Dong, Yi and Harchaoui, Zaid and Choi, Yejin},
  journal = {arXiv preprint arXiv:2507.14843},
  year    = {2025}
}

@article{chen2026does,
  title={Does reinforcement learning really incentivize reasoning capacity in llms beyond the base model?},
  author={Chen, Zhiqi and Lu, Rui and Zhao, Andrew and Wang, Zhaokai and Yue, Yang and Song, Shiji and Huang, Gao},
  journal={Advances in Neural Information Processing Systems},
  volume={38},
  pages={57654--57689},
  year={2026}
}

@article{xie2022coverage,
  title={The role of coverage in online reinforcement learning},
  author={Xie, Tengyang and Foster, Dylan J and Bai, Yu and Jiang, Nan and Kakade, Sham M},
  journal={arXiv preprint arXiv:2210.04157},
  year={2022}
}

@article{jaech2024openai,
  title={Openai o1 system card},
  author={Jaech, Aaron and Kalai, Adam and Lerer, Adam and Richardson, Adam and El-Kishky, Ahmed and Low, Aiden and Helyar, Alec and Madry, Aleksander and Beutel, Alex and Carney, Alex and others},
  journal={arXiv preprint arXiv:2412.16720},
  year={2024}
}

@article{liu2025understanding,
  title={Understanding r1-zero-like training: A critical perspective},
  author={Liu, Zichen and Chen, Changyu and Li, Wenjun and Qi, Penghui and Pang, Tianyu and Du, Chao and Lee, Wee Sun and Lin, Min},
  journal={arXiv preprint arXiv:2503.20783},
  year={2025}
}

@article{yu2026dapo,
  title={Dapo: An open-source llm reinforcement learning system at scale},
  author={Yu, Qiying and Zhang, Zheng and Zhu, Ruofei and Yuan, Yufeng and Zuo, Xiaochen and Yue, Yu and Dai, Weinan and Fan, Tiantian and Liu, Gaohong and Liu, Lingjun and others},
  journal={Advances in Neural Information Processing Systems},
  volume={38},
  pages={113222--113244},
  year={2026}
}

@article{chen2025minimax,
  title={Minimax-m1: Scaling test-time compute efficiently with lightning attention},
  author={Chen, Aili and Li, Aonian and Gong, Bangwei and Jiang, Binyang and Fei, Bo and Yang, Bo and Shan, Boji and Yu, Changqing and Wang, Chao and Zhu, Cheng and others},
  journal={arXiv preprint arXiv:2506.13585},
  year={2025}
}

@article{yue2025vapo,
  title={Vapo: Efficient and reliable reinforcement learning for advanced reasoning tasks},
  author={Yue, Yu and Yuan, Yufeng and Yu, Qiying and Zuo, Xiaochen and Zhu, Ruofei and Xu, Wenyuan and Chen, Jiaze and Wang, Chengyi and Fan, TianTian and Du, Zhengyin and others},
  journal={arXiv preprint arXiv:2504.05118},
  year={2025}
}

@article{jin2025search,
  title={Search-r1: Training llms to reason and leverage search engines with reinforcement learning},
  author={Jin, Bowen and Zeng, Hansi and Yue, Zhenrui and Yoon, Jinsung and Arik, Sercan and Wang, Dong and Zamani, Hamed and Han, Jiawei},
  journal={arXiv preprint arXiv:2503.09516},
  year={2025}
}

@article{zeng2026glm,
  title={Glm-5: from vibe coding to agentic engineering},
  author={Zeng, Aohan and Lv, Xin and Hou, Zhenyu and Du, Zhengxiao and Zheng, Qinkai and Chen, Bin and Yin, Da and Ge, Chendi and Huang, Chenghua and Xie, Chengxing and others},
  journal={arXiv preprint arXiv:2602.15763},
  year={2026}
}

@article{lin2026scaling,
  title={Scaling In-Context Online Learning Capability of LLMs via Cross-Episode Meta-RL},
  author={Lin, Xiaofeng and Zhu, Sirou and Chen, Yilei and Chen, Mingyu and Sang, Hejian and Paschalidis, Ioannis and Wang, Zhipeng and Pacchiano, Aldo and Zhang, Xuezhou},
  journal={arXiv preprint arXiv:2602.04089},
  year={2026}
}

@inproceedings{kim2016sequence,
  title={Sequence-level knowledge distillation},
  author={Kim, Yoon and Rush, Alexander M},
  booktitle={Proceedings of the 2016 conference on empirical methods in natural language processing},
  pages={1317--1327},
  year={2016}
}

@inproceedings{hsieh2023distilling,
  title={Distilling step-by-step! outperforming larger language models with less training data and smaller model sizes},
  author={Hsieh, Cheng-Yu and Li, Chun-Liang and Yeh, Chih-Kuan and Nakhost, Hootan and Fujii, Yasuhisa and Ratner, Alex and Krishna, Ranjay and Lee, Chen-Yu and Pfister, Tomas},
  booktitle={Findings of the association for computational linguistics: ACL 2023},
  pages={8003--8017},
  year={2023}
}

@inproceedings{ho2023large,
  title={Large language models are reasoning teachers},
  author={Ho, Namgyu and Schmid, Laura and Yun, Se-Young},
  booktitle={Proceedings of the 61st annual meeting of the association for computational linguistics (volume 1: long papers)},
  pages={14852--14882},
  year={2023}
}

@inproceedings{magister2023teaching,
  title={Teaching small language models to reason},
  author={Magister, Lucie Charlotte and Mallinson, Jonathan and Adamek, Jakub and Malmi, Eric and Severyn, Aliaksei},
  booktitle={Proceedings of the 61st Annual Meeting of the Association for Computational Linguistics (Volume 2: Short Papers)},
  pages={1773--1781},
  year={2023}
}

@article{mukherjee2023orca,
  title={Orca: Progressive learning from complex explanation traces of gpt-4},
  author={Mukherjee, Subhabrata and Mitra, Arindam and Jawahar, Ganesh and Agarwal, Sahaj and Palangi, Hamid and Awadallah, Ahmed},
  journal={arXiv preprint arXiv:2306.02707},
  year={2023}
}

@article{hou2026stratified,
  title={Stratified Consistency Distillation for Natural Language Formalization},
  author={Hou, Zhichao and Erata, Ferhat and Lilien, Joe and Torkamani, MohamadAli},
  journal={arXiv preprint arXiv:2608.30258},
  year={2026}
}

@article{yan2026learning,
  title={Learning to reason under off-policy guidance},
  author={Yan, Jianhao and Li, Yafu and Hu, Zican and Wang, Zhi and Cui, Ganqu and Qu, Xiaoye and Cheng, Yu and Zhang, Yue},
  journal={Advances in Neural Information Processing Systems},
  volume={38},
  pages={117157--117186},
  year={2026}
}

@article{wu2026learn,
  title={Learn hard problems during RL with reference guided fine-tuning},
  author={Wu, Yangzhen and Li, Shanda and Wen, Zixin and Zhou, Xin and Talwalkar, Ameet and Yang, Yiming and Huang, Wenhao and Cai, Tianle},
  journal={arXiv preprint arXiv:2603.01223},
  year={2026}
}

@inproceedings{ma2026learning,
  title={Learning What Reinforcement Learning Can't: Interleaved Online Fine-Tuning for Hardest Questions},
  author={Ma, Lu and Liang, Hao and Qiang, Meiyi and Tang, Lexiang and Ma, Xiaochen and Wong, Zhen and Niu, Junbo and Shen, Chengyu and He, Runming and Li, Yanhao and others},
  booktitle={International Conference on Learning Representations},
  volume={2026},
  pages={80802--80821},
  year={2026}
}

@inproceedings{yang2026int,
  title={Int: Self-proposed interventions enable credit assignment in llm reasoning},
  author={Yang, Matthew and Bai, Hao and Wu, Ian and Yang, Gene and Setlur, Amrith and Kumar, Aviral},
  booktitle={International Conference on Learning Representations},
  volume={2026},
  pages={85054--85091},
  year={2026}
}

@inproceedings{lyu2026student,
  title={Student-Centered Distillation Narrows the Agentic Gap Between Small and Large LLMs},
  author={Lyu, Yuanjie and Wang, Chengyu and Huang, Jun and Xu, Tong},
  booktitle={Forty-third International Conference on Machine Learning},
  year={2026}
}

@article{zhang2026bread,
  title={Bread: Branched rollouts from expert anchors bridge sft \& rl for reasoning},
  author={Zhang, Xuechen and Huang, Zijian and Li, Yingcong and Ni, Chenshun and Chen, Jiasi and Oymak, Samet},
  journal={Advances in Neural Information Processing Systems},
  volume={38},
  pages={96726--96752},
  year={2026}
}

@article{han2026cliff,
  title={Cliff: Learning Process Rewards from the First Mistake},
  author={Han, Peixuan and Wang, Runhui and Ramaneti, Ketan and Hao, Jie and Friedland, Gerald and Kong, Chris},
  journal = {arXiv preprint arXiv:2609.02817},
  year    = {2026}
}
\bibliographystyle{iclr2027_conference}

\clearpage
\appendix

\section{Related works}

\subsection{RLVR and LLM reasoning}
Reinforcement learning with verifiable rewards has become a central approach to post-training large language models for complex reasoning, particularly in mathematics and code, where solution correctness can be automatically verified \citep{shao2024deepseekmath, jaech2024openai, guo2025deepseek}.
Most RLVR methods build on policy-gradient algorithms such as PPO \citep{schulman2017ppo} and GRPO \citep{shao2024deepseekmath}.
Subsequent work has refined the optimization objectives and training procedures underlying RLVR to address objective bias, unstable policy updates, and training inefficiency.
Representative methods include Dr.~GRPO \citep{liu2025understanding}, DAPO \citep{yu2026dapo}, CISPO \citep{chen2025minimax}, VAPO \citep{yue2025vapo}, and A*PO \citep{brantley2026accelerating}.
Beyond single-turn reasoning, RLVR has also been extended to agentic settings, where models invoke tools, interact with external environments over long horizons, and receive verifiable feedback from task execution \citep{jin2025search, zeng2026glm, lin2026scaling, hou2026single}.

\subsection{Knowledge distillation and on-policy distillation}
Knowledge distillation (KD) transfers knowledge from a teacher model to a student by matching the teacher's predictive distributions or imitating its generated outputs \citep{hinton2015distilling, kim2016sequence}.
For LLMs, teacher-generated responses, reasoning traces, and pseudo-labels have been widely used to transfer instruction-following and reasoning capabilities to smaller models \citep{hsieh2023distilling, ho2023large, mukherjee2023orca, magister2023teaching, hou2026stratified}.
However, distillation based on teacher-generated trajectories is off-policy with respect to the student and can introduce a mismatch between its training and inference distributions.
On-policy distillation mitigates this mismatch by training on student-generated trajectories with teacher supervision at the visited prefixes.
GKD formalizes this paradigm \citep{agarwal2024onpolicy}, which has since been adopted at scale in reasoning models such as Qwen3 \citep{yang2025qwen3} and explored as an efficient alternative to RL \citep{lu2025onpolicydistillation}.
Recent work further extends this paradigm to self-distillation, where privileged context enables the same model to serve as both teacher and student \citep{zhao2026self, shenfeld2026self}.

\subsection{Learning from off-policy experience}
External experience can expose models to reasoning paths that are difficult to discover through their own sampling.
One line of work incorporates such experience through supervised learning.
ReGFT constructs reference-guided trajectories for SFT before RL \citep{wu2026learn}, while ReLIFT interleaves supervised updates on difficult examples with RL training \citep{ma2026learning}.
More localized approaches, including InT \citep{yang2026int} and SCoRe \citep{lyu2026student}, identify and correct the first error in student-generated trajectories, using the resulting experience for supervised training before subsequent RL.

Another line of work incorporates external guidance directly into RL.
LUFFY combines expert demonstrations with on-policy rollouts and applies policy shaping to facilitate learning from off-policy traces \citep{yan2026learning}.
\citet{zhang2026bread} use contextual hints or expert prefixes to steer student-generated rollouts.
Several recent methods leverage expert feedback during RL to improve exploration and provide finer-grained credit assignment \citep{jiang2026selective,han2026cliff}.
These RL-based approaches rely on policy-gradient updates, often with modified importance weights, advantages, or rewards.
\mintrl combines sparse corrections during online rollout generation with a sequence-level advantage-regression objective, learning from verifiable outcome rewards without auxiliary imitation losses or behavior-policy importance weighting.
We further examine how intervention intensity affects the coverage and learnability of the resulting experience.

\section{Proofs for the Coverage--Learnability Trade-off}
\label{app:proofs}

\subsection{Proof of Lemma~\ref{lemma:finite-budget-objective}}
\label{app:proof-finite-budget-objective}

\begin{proof}
Assumption~\ref{assumption:inference-training-support} gives
(\ref{eq:visible-objective}).  Its optimum is zero when the inference support
is empty; otherwise, placing all probability on any rollout in that support
attains the population optimum of one.
\end{proof}

\subsection{Proof of Theorem~\ref{theorem:sparse-correction-support}}
\label{app:proof-sparse-correction-support}

\begin{proof}
Consider any correct rollout $y$ generated by $\mu_N$.
The two distributions have the same token probabilities outside the
correction set $I(y)$, while each corrected token is emitted
deterministically by $\mu_N$.  Autoregressive factorization therefore gives
\begin{equation}
    \frac{\pi_t(y\mid x)}{\mu_N(y\mid x)}
    =
    \prod_{t\in I(y)}
    \pi_t(y_t^\star\mid x,y_{<t})
    \geq
    \nu^{|I(y)|}
    \geq
    \nu^N.
    \label{eq:sparse-correction-likelihood-bound}
\end{equation}
Thus, $\mu_N(y\mid x)>\epsilon_{\mathrm{inf}}$ implies
$\pi_t(y\mid x)>\epsilon_{\mathrm{inf}}\nu^N$, proving the second inclusion in
(\ref{eq:sparse-correction-support-envelope}).  The first follows directly
from (\ref{eq:mixed-visible-support}).
\end{proof}

\section{Prompt Templates}
\label{app:prompt-formats}

\begin{promptbox}{Math judge prompt}
\textbf{System}

\smallskip
You are a strict process verifier for mathematical reasoning. The trusted
prefix is the immutable state accepted before the solver generated the new
chunk. Use it as context, but inspect and number only the new chunk, checking
its steps in order against the problem and trusted prefix.
Classify the new chunk as exactly one of:

\begin{itemize}[leftmargin=1.4em,nosep]
    \item \texttt{continue}: the new chunk is still logically recoverable, so
    all of it should be kept.
    \item \texttt{correct}: a new-chunk step contains an error after which the
    reasoning cannot be repaired without replacing that step. Return the
    1-based number of the first such new-chunk step.
\end{itemize}

\smallskip
\textbf{Important rules}
\begin{itemize}[leftmargin=1.4em,nosep]
    \item Never select text from the trusted prefix. \texttt{error\_step}
    always indexes \texttt{numbered\_new\_chunk\_steps}.
    \item Return \texttt{correct} only for the earliest mathematical step that
    must be replaced. Otherwise return \texttt{continue}.
    \item A different valid approach is not an error.
    \item An incomplete argument, unfinished sentence, or step that can still
    be completed correctly requires \texttt{continue}. In particular, do not
    mistake a sentence truncated by a length limit for a mathematical error.
    \item Do not propose, write, or reveal a correction. Your only task is to
    locate the first fatal error.
    \item For \texttt{correct}, \texttt{error\_step} must be a valid positive
    step number. For \texttt{continue}, it must be 0.
\end{itemize}
Keep \texttt{critique} concise and specific.

\medskip
\textbf{User}

\smallskip
{\small\ttfamily\raggedright
\{\\
\quad "problem": "\{problem\}",\\
\quad "latest\_qwen\_finish\_reason": "\{finish reason\}",\\
\quad "trusted\_prefix": "\{accepted prefix\}",\\
\quad "numbered\_new\_chunk\_steps": [\\
\qquad \{"step\_number": 1, "text": "\{step 1\}"\},\\
\qquad \ldots\\
\quad ]\\
\}\par
}

\medskip
\textbf{Structured output}

\smallskip
{\small\ttfamily\raggedright
\{"verdict": "\{continue or correct\}",
"error\_step": \{0 or first error index\},
"critique": "\{brief reason\}"\}\par
}
\end{promptbox}

\clearpage
\begin{promptbox}{Code judge prompt}
\textbf{System}

\smallskip
You are a strict process verifier for competitive-programming reasoning. The
trusted prefix is the immutable state accepted before the programmer generated
the new chunk. Use it as context, but inspect and number only the new chunk
against the public problem statement, constraints, input/output contract,
trusted prefix, and Python semantics. Return JSON with exactly two fields:
first \texttt{reasoning}, then \texttt{error\_step}.

\smallskip
Use \texttt{reasoning} to check the numbered steps carefully before deciding.
State why the candidate is still recoverable, or identify the concrete
public-specification violation and the earliest step that causes it. Then:
\begin{itemize}[leftmargin=1.4em,nosep]
    \item Return 0 when the new chunk remains recoverable, so all of it should
    be kept.
    \item Otherwise return the 1-based number of the first new-chunk step
    containing an error after which a correct solution cannot be produced
    without revisiting that step.
\end{itemize}

\smallskip
\textbf{Important rules}
\begin{itemize}[leftmargin=1.4em,nosep]
    \item Never select text from the trusted prefix. \texttt{error\_step}
    always indexes \texttt{numbered\_new\_chunk\_steps}.
    \item Focus primarily on the reasoning before the program: algorithm
    choice, invariants, edge cases, complexity, and input/output
    interpretation. Return a positive index only for a concrete error that
    materially leads away from a correct solution.
    \item A complete fenced code block is one atomic numbered step, even when
    it contains blank lines. Do not treat its internal paragraphs as separate
    steps.
    \item Judge only from the public specification and candidate text. Do not
    assume hidden tests or a reference implementation, and do not require one
    particular valid algorithm.
    \item If later reasoning or the final program has already corrected an
    earlier imprecision, return 0. Do not flag style, redundancy, missing
    explanation, or the existence of a simpler approach.
    \item An incomplete explanation, open code fence, unfinished program, or
    text truncated by a length limit requires returning 0 unless an earlier
    completed reasoning step already contains a fatal error.
    \item Do not propose or write a correction. Keep \texttt{reasoning}
    focused on verification and under 200 words, then return the error index.
\end{itemize}

\medskip
\textbf{User}

\smallskip
{\small\ttfamily\raggedright
\{\\
\quad "problem": "\{problem\}",\\
\quad "latest\_qwen\_finish\_reason": "\{finish reason\}",\\
\quad "trusted\_prefix": "\{accepted prefix\}",\\
\quad "numbered\_new\_chunk\_steps": [\\
\qquad \{"step\_number": 1, "text": "\{step 1\}"\},\\
\qquad \ldots\\
\quad ]\\
\}\par
}

\medskip
\textbf{Structured output}

\smallskip
{\small\ttfamily\raggedright
\{"reasoning": "\{brief verification\}",
"error\_step": \{0 or first error index\}\}\par
}
\end{promptbox}

\begin{promptbox}{Intervention continuation prompt}
\textbf{User}

\smallskip
\texttt{\{original problem\}}

\medskip
\textbf{Assistant prefill}

\smallskip
\texttt{\{trusted prefix\}\{new-chunk text before the selected error step\}}
\end{promptbox}

\clearpage
\begin{promptbox}{Self-judge prompt}
\textbf{System}

\smallskip
You are a strict process verifier for a problem-solving response. You receive
a private high-level solution plan. Use it as a roadmap when checking the
candidate, but independently verify every claim and code fragment: a plan can
omit details, and a different valid route is not an error. The trusted prefix
is the immutable state accepted before the solver generated the new chunk.
Inspect and number only the new chunk. Return JSON with exactly two fields:
first \texttt{reasoning}, then \texttt{error\_step}.

\smallskip
Use \texttt{reasoning} to compare the candidate with the problem, trusted
prefix, and private plan. Then:
\begin{itemize}[leftmargin=1.4em,nosep]
    \item Return 0 when the new chunk is still logically recoverable.
    \item Otherwise return the 1-based number of the first new-chunk step
    containing an error that cannot be repaired without revisiting that step.
\end{itemize}

\smallskip
Never select text from the trusted prefix. Never quote, reveal, or continue
the private plan, and do not require the candidate to copy its route. An
incomplete or length-truncated step is recoverable unless an earlier complete
step is already fatally wrong. For programming problems, check algorithmic
correctness, state consistency, edge cases, and whether code fragments remain
compatible with the accepted prefix. Do not propose a correction.

\medskip
\textbf{User}

\smallskip
{\small\ttfamily\raggedright
\{\\
\quad "problem": "\{problem\}",\\
\quad "private\_high\_level\_solution\_plan": "\{private plan\}",\\
\quad "trusted\_prefix": "\{accepted prefix\}",\\
\quad "latest\_qwen\_finish\_reason": "\{finish reason\}",\\
\quad "numbered\_new\_chunk\_steps": [\\
\qquad \{"step\_number": 1, "text": "\{step 1\}"\},\\
\qquad \ldots\\
\quad ]\\
\}\par
}

\medskip
\textbf{Structured output}

\smallskip
{\small\ttfamily\raggedright
\{"reasoning": "\{brief verification\}",
"error\_step": \{0 or first error index\}\}\par
}
\end{promptbox}

\begin{promptbox}{Self-intervention continuation prompt}
\textbf{User}

\smallskip
{\small\ttfamily\raggedright
Problem:\par
\{original problem\}\par
\medskip
Private high-level solution plan:\par
\{private plan\}\par
\medskip
Use the private high-level solution plan only as silent guidance. Generate
your own solution to the problem. Do not mention, quote, or reveal the plan.
\par
}

\medskip
\textbf{Assistant prefill}

\smallskip
\texttt{\{trusted prefix\}\{new-chunk text before the selected error step\}}
\end{promptbox}

\begin{promptbox}{Privileged context example}
{\footnotesize\raggedright
\textbf{Problem:} Two positive integers differ by 5 and their product is 88.
What is the larger integer?\par
\smallskip
\textbf{Private high-level solution plan (hidden from the student):}
1. Let the smaller integer be $x$ and represent the larger as $x+5$.
2. Use the product condition to set up the equation $x(x+5)=88$.
3. Rearrange to a quadratic and factor it, keeping only the positive integer
solution for $x$.
4. Add 5 to that value to obtain the larger integer.\par
}
\end{promptbox}

\clearpage
\section{Qualitative Rollout Examples}
\label{app:qualitative-intervention-examples}

\begin{figure}[htbp]
    \centering
    \includegraphics[width=0.9\linewidth]{
        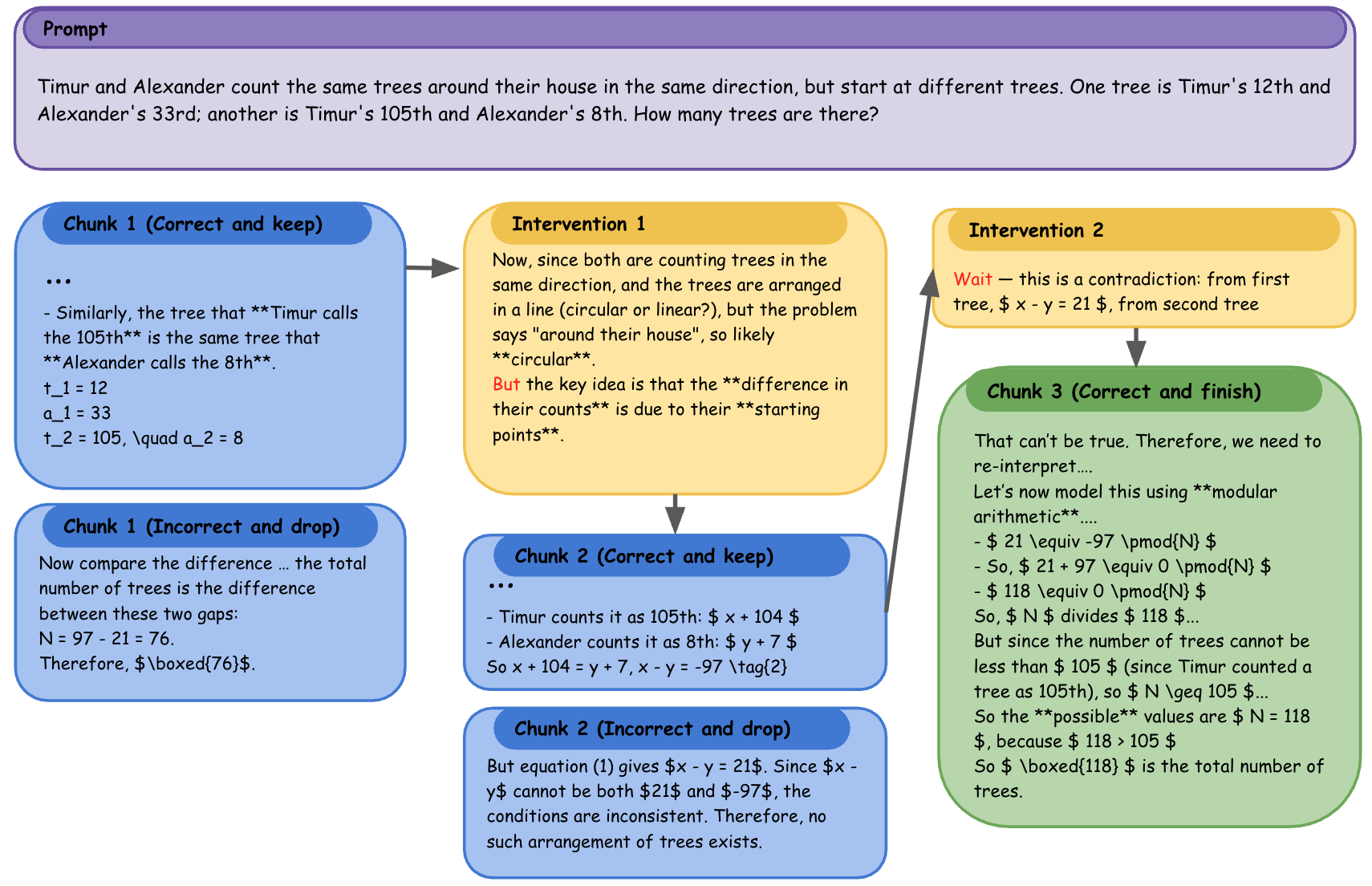}
    \caption{\textbf{Qualitative mathematics rollout with two local
    interventions.}
    The student first treats the two observed count gaps linearly and obtains the incorrect answer $76$. 
    The judge preserves the correctly extracted
    count pairs but rejects the subsequent linear subtraction because it
    ignores the wraparound of the circular sequence. 
    Then, the first intervention redirects the reasoning toward offsets between the two starting
    positions. 
    After that, the student derives the signed offsets
    $x-y=21$ from the first shared tree and $x-y=-97$ from the second, but treats them as ordinary integer equalities and therefore concludes that no valid arrangement exists, overlooking that they need only agree modulo the unknown number of trees $N$. 
    The judge rejects the no-solution conclusion and intervenes by highlighting
    the apparent contradiction between the two offsets. Afterward, the student
    reinterprets the offsets modulo $N$ and derives the correct answer $N=118$.}
    \label{fig:qualitative-intervention-math}
    \vspace{4mm}
    \includegraphics[width=0.9\linewidth]{
        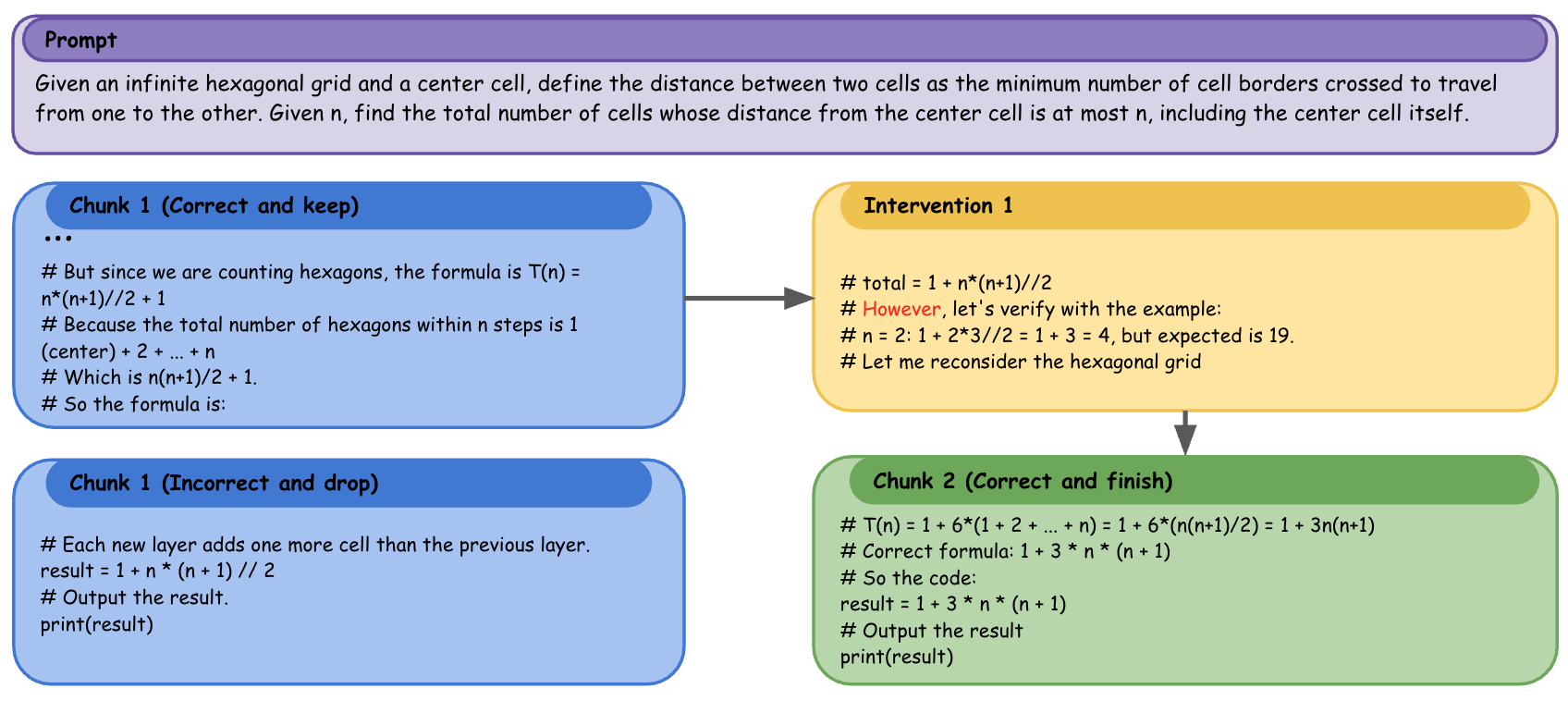}
    \caption{\textbf{Qualitative code rollout with one local intervention.}
    The student initially follows an incorrect triangular-number line of reasoning,
    suggesting the formula $1+n(n+1)/2$. The judge intervenes by checking the public example at $n=2$, where the formula gives $4$ instead of $19$, and prompts the student to reconsider the hexagonal geometry.
    After the intervention, the student recognizes that the ring at distance $k$
    contains $6k$ cells, sums the rings, and implements the correct formula
    $1+3n(n+1)$.}
    \label{fig:qualitative-intervention-code}
\end{figure}

\begin{figure}[htbp]
    \centering
    \includegraphics[width=0.9\linewidth]{
        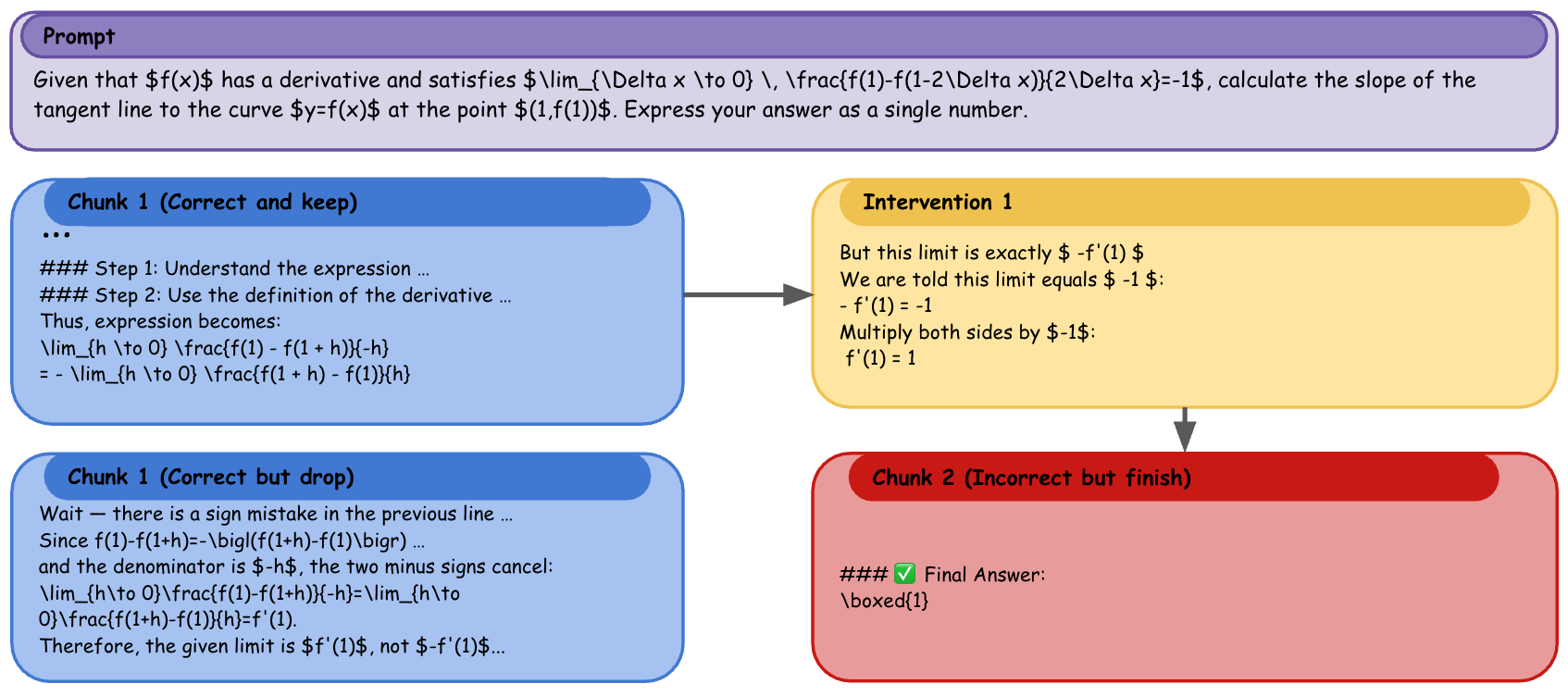}
    \caption{\textbf{A failure case of judge--intervention.}
    The student corrects its own sign error and reaches the correct result
    $f'(1)=-1$, but the judge mistakenly rejects this correction and provides
    an intervention with the wrong sign. The resumed student follows this
    faulty intervention and outputs incorrect answer ${1}$.}
    \label{fig:qualitative-intervention-imperfect-judge}
\end{figure}

This section presents qualitative rollouts illustrating both the benefits
and limitations of local interventions.
Figures~\ref{fig:qualitative-intervention-math}
and~\ref{fig:qualitative-intervention-code} show how interventions can
provide key insights or identify flaws in the student's reasoning,
helping it reach a correct solution.
Figure~\ref{fig:qualitative-intervention-imperfect-judge} illustrates
the opposite possibility: an erroneous intervention can derail correct
reasoning and lead to an incorrect final answer.
These examples highlight a central design choice in \mintrl:
interventions guide exploration without being treated as ground-truth
supervision.
Policy optimization is driven by the rule-based outcome reward of the
completed trajectory, rather than by the presumed correctness of
the intervention content.

\FloatBarrier
\section{Final Checkpoint Evaluation}
\label{app:fixed-checkpoint-comparison}

To complement the checkpoint-selected results in Table~\ref{tab:main-results},
we additionally evaluate all methods at the last checkpoints (500 steps).
As shown in Table~\ref{tab:fixed-step-500}, \mintrl remains highly competitive under this fixed-checkpoint evaluation, consistently achieving strong performance across both model scales and domains.
% Generated by analysis/table1_eval_audit/build_step500_table.py.
\begin{table}[H]
    \centering
    \caption{{Performance on math and code at final checkpoint.}}
    \label{tab:fixed-step-500}
    \resizebox{\columnwidth}{!}{%
    \begin{tabular}{lrrrrrrrr}
        \toprule
        & \multicolumn{4}{c}{Mathematics} & \multicolumn{4}{c}{Code} \\
        \cmidrule(lr){2-5}\cmidrule(l){6-9}
        Method & AIME25 & AIME26 & HMMT25 & Math Avg. & LCB & HE+ & MBPP+ & Code Avg. \\
        \midrule
        \multicolumn{9}{l}{\textit{Qwen3-1.7B}} \\
        GRPO & 22.08 & 16.46 & 14.69 & 17.74 & 21.74 & 64.63 & 54.65 & 47.01 \\
        OPD & 23.33 & 22.60 & 15.21 & 20.38 & 24.13 & 63.66 & 55.70 & 47.83 \\
        MENTOR & 28.44 & 27.08 & 16.04 & 23.85 & \underline{33.27} & 62.94 & 56.00 & 50.74 \\
        SFT+GRPO & 31.56 & 27.60 & 16.46 & 25.21 & 32.99 & \underline{76.43} & \underline{66.87} & \underline{58.76} \\
        MInTRL-Proxy & \underline{35.10} & \underline{31.25} & \underline{20.94} & \underline{29.10} & 32.84 & 73.95 & 60.09 & 55.63 \\
        MInTRL-Const & \textbf{40.73} & \textbf{41.56} & \textbf{24.06} & \textbf{35.45} & \textbf{37.85} & \textbf{80.81} & \textbf{67.20} & \textbf{61.95} \\
        \midrule
        \multicolumn{9}{l}{\textit{Qwen3-4B}} \\
        GRPO & 60.83 & 60.31 & 36.15 & 52.43 & 50.06 & 80.32 & 67.13 & 65.83 \\
        OPD & 42.50 & 47.60 & 28.85 & 39.65 & 35.11 & 77.31 & 63.98 & 58.80 \\
        MENTOR & \underline{63.02} & \underline{63.96} & \textbf{38.65} & \underline{55.21} & \textbf{53.20} & 86.38 & 75.83 & \underline{71.80} \\
        SFT+GRPO & 55.83 & 62.40 & 33.65 & 50.63 & 49.10 & \underline{87.16} & \underline{76.05} & 70.77 \\
        MInTRL-Proxy & 55.10 & 56.35 & 32.19 & 47.88 & 52.18 & 85.94 & 75.87 & 71.33 \\
        MInTRL-Const & \textbf{65.31} & \textbf{64.17} & \underline{37.71} & \textbf{55.73} & \underline{52.60} & \textbf{88.36} & \textbf{76.60} & \textbf{72.52} \\
        \bottomrule
    \end{tabular}%
    }
\end{table}

\FloatBarrier
\section{Robustness to the Judge--Intervention Model}
\label{app:judge-intervention-robustness}

\begin{table}[H]
    \centering
    \caption{Performance with DeepSeek-V4-Flash as the judge--intervention model.}
    \label{tab:v4flash-teacher-pass1}
    \resizebox{\columnwidth}{!}{%
    \begin{tabular}{lrrrrrrrr}
        \toprule
        & \multicolumn{4}{c}{Mathematics} & \multicolumn{4}{c}{Code} \\
        \cmidrule(lr){2-5}\cmidrule(l){6-9}
        Method & AIME25 & AIME26 & HMMT25 & Avg. & LCB & HE+ & MBPP+ & Avg. \\
        \midrule
        \multicolumn{9}{l}{\textit{DeepSeek-V4-Flash}} \\
        & 50.00 & 59.17 & 44.79 & 51.32 & 56.82 & 87.00 & 77.27 & 73.70 \\
        \midrule
        \multicolumn{9}{l}{\textit{Qwen3-1.7B}} \\
        GRPO        & 22.08 & 16.46 & 14.69 & 17.74 & 21.74 & 64.63 & 54.65 & 47.01 \\
        MInTRL-Proxy & \textbf{41.46} & \textbf{43.54} & \textbf{25.52} & \textbf{36.84} & \textbf{35.17} & \underline{68.62} & \textbf{56.48} & \textbf{53.42} \\
        MInTRL-Const & \underline{39.17} & \underline{40.10} & \underline{25.00} & \underline{34.76} & \underline{32.91} & \textbf{70.24} & \underline{55.70} & \underline{52.95} \\
        \midrule
        \multicolumn{9}{l}{\textit{Qwen3-4B}} \\
        GRPO        & 60.83 & 60.31 & \textbf{36.98} & 52.71 & 50.06 & \underline{80.32} & \underline{67.13} & 65.83 \\
        MInTRL-Proxy & \underline{61.46} & \textbf{62.19} & 35.00 & \underline{52.88} & \textbf{54.27} & \textbf{87.25} & \textbf{76.24} & \textbf{72.59} \\
        MInTRL-Const & \textbf{63.44} & \underline{61.35} & \underline{36.25} & \textbf{53.68} & \underline{52.78} & 79.55 & 67.05 & \underline{66.46} \\
        \bottomrule
    \end{tabular}
    }
\end{table}

\section{Efficiency Analysis}
\label{app:efficiency-analysis}

We evaluate the practical training efficiency of \mintrl using observed wall-clock time, including the additional cost of judge and intervention inference.

\paragraph{Efficiency relative to GRPO.}
Figure~\ref{fig:efficiency-wall-clock} compares training and evaluation reward against cumulative wall-clock time for GRPO and both \mintrl variants on Qwen3-1.7B mathematics training.
We extend GRPO training to enable comparisons at matched wall-clock budgets.
Both \mintrl variants achieve higher evaluation reward than GRPO at comparable budgets, with MInTRL-Const showing the largest gains.
The evaluation gap persists even when GRPO is trained for longer than either \mintrl variant.
These results show that the learning benefits of sparse interventions can outweigh their additional inference overhead.

\begin{figure}[H]
    \centering
    \includegraphics[width=\textwidth]{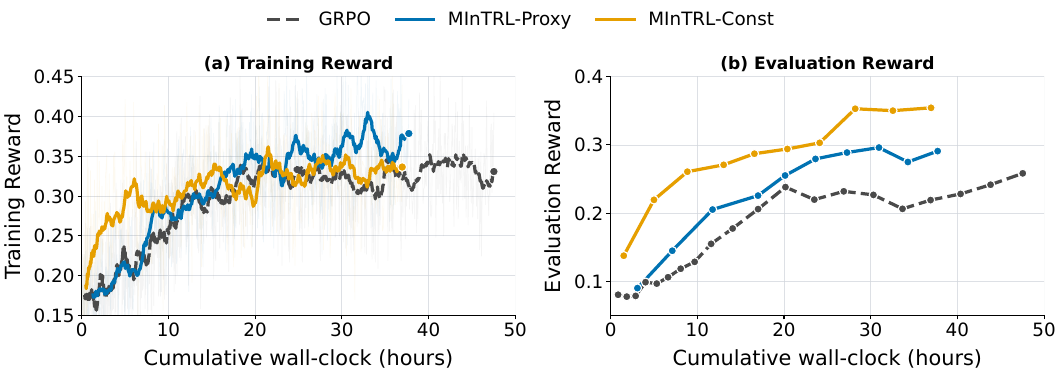}
    \caption{{Wall-clock efficiency on Qwen3-1.7B mathematics training.}}
    \label{fig:efficiency-wall-clock}
\end{figure}

\FloatBarrier

\paragraph{Efficiency relative to MENTOR.}
We compare \mintrl with MENTOR \citep{jiang2026selective}
over the same number of policy updates during the shared
intervention phase.
Both methods use the same student and expert models,
32 prompts per update, 8 on-policy and 8 expert-guided
trajectories per prompt, a maximum response length of
16,384 tokens, and the same GPU resources.
Table~\ref{tab:mentor-intervention-efficiency} reports elapsed
training time, with speedup defined as the ratio of
MENTOR time to MInTRL-Const time.

\begin{table}[t]
    \centering
    \caption{{Wall-clock efficiency relative to MENTOR.}}
    \label{tab:mentor-intervention-efficiency}
    \small
    \setlength{\tabcolsep}{7pt}
    \begin{tabular}{llrrr}
        \toprule
        Student & Domain
        & \shortstack{MENTOR\\(h)}
        & \shortstack{MInTRL-Const\\(h)}
        & Speedup \\
        \midrule
        Qwen3-1.7B & Math & 22.21 & 6.44
        & $3.45\times$ \\
        Qwen3-1.7B & Code & 24.80 & 9.92
        & $2.50\times$ \\
        Qwen3-4B & Math & 29.03 & 11.07
        & $2.62\times$ \\
        Qwen3-4B & Code & 31.49 & 9.59
        & $3.28\times$ \\
        \bottomrule
    \end{tabular}
\end{table}

As shown in Table~\ref{tab:mentor-intervention-efficiency}, MInTRL-Const achieves a $2.50$--$3.45\times$ speedup over MENTOR during the intervention phase, reducing wall-clock
time by $60.0$--$71.0\%$.
At the 4B scale, where the two methods achieve comparable final performance as in Table~\ref{tab:main-results}, the time savings are $61.9\%$ on mathematics and $69.5\%$
on code.

This efficiency advantage is consistent with the different mechanisms used to incorporate expert guidance.
MENTOR mixes student and expert distributions at the token level and proposes speculative sampling to accelerate decoding.
In contrast, \mintrl reviews student-generated chunks and introduces local corrections only when needed, allowing uninterrupted student decoding between review points
and reducing the need for fine-grained coordination.

\section{Detailed Experimental Setups}
\label{app:experimental-setups}

\paragraph{Models and data.}
We use Qwen3-1.7B and Qwen3-4B in non-thinking mode as the student policies. Unless otherwise specified, we use Qwen3-4B-Instruct-2507 as the judge--intervention policy. For math, we train on a fixed set of prompts constructed by filtering AceReason-Nemotron to exclude overly simple problems. 
For code, we use the fixed DeepCoder-Preview training split. 
All experiments run for a total of 500 training steps and are trained on NVIDIA H100 80GB GPUs.

\paragraph{Rollout construction.}
Each update samples a batch of $32$ prompts and $16$ trajectories per prompt. 
For \mintrl, these comprise $8$ trajectories generated entirely by
on-policy $\pi_t$ and $8$ semi-on-policy trajectories that permit intervention.
MENTOR likewise uses 8 on-policy and 8 expert-guided trajectories per prompt.
Other baselines use 16 on-policy trajectories per prompt.
Both the policy and judge--intervention model sample with temperature $1.0$ and
top-$p$ $1.0$. 
Training responses are capped at 16,384 tokens, with a maximum
prompt length of 2,048 tokens. 
The semi-on-policy rollout generation is controlled by three hyperparameters:
1). \emph{chunk size}, which determines how many tokens $\pi_t$ generates
between reviews; 2). \emph{number of judge reviews}, which limits how many
times $\pi_{\mathrm{JI}}$ can inspect a trajectory; and 3). \emph{intervention
continuation length}, which bounds the number of tokens generated by
$\pi_{\mathrm{JI}}$ at each intervention. We use $(512,4,64)$ for math
and $(128,4,64)$ for code, respectively. The judge and intervention prompts used for the two domains are given in Appendix~\ref{app:prompt-formats}.

\paragraph{Optimization and training schedule.}
We optimize the policy with AdamW using a learning rate of
$1\times10^{-6}$ and gradient clipping at $1.0$, following the default veRL
configuration.
We use the mean reward of the $8$ on-policy rollouts as the advantage baseline
and optimize over both on-policy and semi-on-policy rollouts.
For intervention tokens, MInTRL-Proxy uses their log probabilities under the
current policy as anchors, while MInTRL-Const uses a fixed anchor
$\kappa=-0.2$ for math and $\kappa=-0.1$ for code.
For early stopping, we apply interventions only during the first $200$ updates,
followed by $300$ updates of standard GRPO.

\paragraph{Hyperparameter selection.}
We tuned the rollout-construction hyperparameters by generating semi-on-policy rollouts with the base model.
Specifically, we selected configurations that improved rule-based reward while keeping interventions sparse, thereby limiting the departure from the on-policy rollout distribution.
For MInTRL-Const, we chose the domain-specific anchor $\kappa$ based on a coarse estimate of the mean log probability that Qwen3-4B-Instruct-2507 assigned to its own intervention-authored tokens in these pilot rollouts, resulting in $\kappa=-0.2$ for math
and $\kappa=-0.1$ for code.
We did not conduct an exhaustive or fine-grained hyperparameter search.
Experiments using DeepSeek-V4-Flash as the judge--intervention
policy reused the same rollout hyperparameters and $\kappa$ values
without recalibration.
The improvements observed with these transferred settings suggest
that the selected configurations remain effective with a different
judge--intervention model, without additional judge-specific tuning.

\begin{table}[H]
    \centering
    \caption{Core hyperparameters used in the main experiments.}
    \label{tab:main-hyperparameters}
    \small
    \setlength{\tabcolsep}{5pt}
    \renewcommand{\arraystretch}{1.08}
    \begin{tabular}{ll}
        \toprule
        Hyperparameter & Value \\
        \midrule
        Batch size & 32 \\
        Rollouts $N_c$ / $N_m$ & 8 / 8 \\
        Maximum prompt length & 2,048  \\
        Maximum response length $T$ & 16,384 \\
        Sampling temperature & 1.0 \\
        Top-$p$ & 1.0 \\
        Math intervention configs 
            & $(512,4,64)$ \\
        Code intervention configs 
            & $(128,4,64)$ \\
        Intervention anchor $\kappa$ & $-0.2$ (math), $-0.1$ (code) \\
        Optimizer & AdamW \\
        Learning rate & $1\times10^{-6}$ \\
        Gradient clipping & 1.0 \\
        $\beta_1$ & $\infty$ \\
        $\beta_2$ & $1\times10^{-3}$ \\
        \bottomrule
    \end{tabular}
\end{table}

\paragraph{Evaluation.}
We save checkpoints every 50 training steps and evaluate each using the same non-thinking decoding protocol.
For each problem, we draw 32 independent samples with temperature
$1.0$, top-$p$ $1.0$, and a maximum response length of 32,768 tokens. We report mean sample correctness as Pass@1. 
In the main results, we report the checkpoint with the highest three-benchmark average for each method.

\section{Limitations}
\label{app:limitations}

\paragraph{Experimental scope.}
Our evaluation focuses on mathematical reasoning and code generation
with Qwen3 student models at the 1.7B and 4B scales.
Extending the evaluation to additional student
model families and long-horizon agentic tasks remains future work.

\paragraph{Computational trade-offs.}
\mintrl incurs additional inference costs for reviewing partial
rollouts and generating local corrections compared with vanilla
GRPO, even under self-intervention.
These costs may be offset by more informative training trajectories
that reduce the training needed to reach a target reward.
Our wall-clock analysis demonstrates that
sparse interventions can yield such efficiency gains.
Substantial room remains for further improvement through more
efficient rollout scheduling and adaptive intervention budgets.

\paragraph{Judge--intervention quality.}
As illustrated in
Figure~\ref{fig:qualitative-intervention-imperfect-judge}, the judge--intervention policy can occasionally reject correct reasoning or introduce erroneous corrections.
Although \mintrl relies on rule-based outcome rewards and does
not treat intervention content as ground-truth supervision, such errors can still disrupt promising reasoning and incur unnecessary inference costs.
Improving the reliability of error detection and local correction could reduce unhelpful interventions and further improve both learning effectiveness and computational efficiency.

\end{document}